\documentclass{article}
\pdfoutput=1

\usepackage{microtype}
\usepackage{graphicx}
\usepackage{subfigure}
\usepackage{booktabs} 

\usepackage{hyperref}

\usepackage[accepted]{icml2020}
\usepackage{xspace}		% for flexible spaces

\usepackage{algorithm}
\usepackage{algorithmic}
\usepackage{amsmath,amssymb,mdframed,amsthm}
\usepackage{cleveref}

\makeatletter
\newtheorem*{rep@theorem}{\rep@title}
\newcommand{\newreptheorem}[2]{%
\newenvironment{rep#1}[1]{%
 \def\rep@title{#2 \ref{##1}}%
 \begin{rep@theorem}}%
 {\end{rep@theorem}}}
\makeatother

\newreptheorem{theorem}{Theorem}
\newreptheorem{proposition}{Proposition}

\newtheorem{lemma}{Lemma}
\newtheorem{theorem}{Theorem}
\newtheorem{corollary}{Corollary}
\newtheorem{proposition}{Proposition}
\newtheorem{remark}{Remark}

\newcommand{\adagrad}{\textsc{Adagrad}\xspace}
\newcommand{\sadam}{\textsc{Sadam}\xspace}

\newcommand{\adam}{\textsc{Adam}\xspace} 
\newcommand{\adamnc}{\textsc{AdamNc}\xspace}

\newcommand{\amsgrad}{\textsc{AMSGrad}\xspace}

\newmdtheoremenv[
outerlinewidth=2,
roundcorner=10 pt,
leftmargin=1,
rightmargin=1,
outerlinecolor=blue!70!black,
innertopmargin=\topskip,
splittopskip=\topskip,
ntheorem=true]{assumption}{Assumption}

\newcommand{\R}{\mathbb{R}}
\newcommand{\lr}[1]{\langle #1\rangle}
\newcommand{\n}[1]{\|#1 \|}
\renewcommand{\a}{\alpha}
\renewcommand{\b}{\beta}

\newcommand{\hv}{\hat v}
\newcommand{\cO}{\mathcal{O}}
\newcommand{\cX}{\mathcal{X}}
\renewcommand{\epsilon}{\varepsilon}
\newenvironment{sproof}{%
  \proof}{\endproof}

\DeclareMathOperator{\diag}{diag}
\DeclareMathOperator*{\argmin}{argmin}

\icmltitlerunning{A new regret analysis for Adam-type algorithms}

\begin{document}

\twocolumn[
\icmltitle{A new regret analysis for Adam-type algorithms}

\icmlsetsymbol{equal}{*}

\begin{icmlauthorlist}
\icmlauthor{Ahmet Alacaoglu}{to}
\icmlauthor{Yura Malitsky}{to}
\icmlauthor{Panayotis Mertikopoulos}{goo}
\icmlauthor{Volkan Cevher}{to}
\end{icmlauthorlist}

\icmlaffiliation{to}{Ecole Polytechnique F\'ed\'erale de Lausanne, Switzerland}
\icmlaffiliation{goo}{Univ. Grenoble Alpes, CNRS, Inria,
LIG, 38000 Grenoble}

\icmlcorrespondingauthor{Ahmet Alacaoglu}{ahmet.alacaoglu@epfl.ch}

\icmlkeywords{Regret,...}

\vskip 0.3in
]

\printAffiliationsAndNotice{}  

\begin{abstract}
In this paper, we focus on a theory-practice gap for Adam and its variants (AMSgrad, AdamNC, etc.).
In practice, these algorithms are used with a constant first-order moment parameter $\beta_{1}$ (typically between $0.9$ and $0.99$).
In theory, regret guarantees for online convex optimization require a rapidly decaying $\beta_{1}\to0$ schedule.
We show that this is an artifact of the standard analysis and propose a novel framework that allows us to derive optimal, data-dependent regret bounds with a constant $\beta_{1}$, without further assumptions.
We also demonstrate the flexibility of our analysis on a wide range of different algorithms and settings.
\end{abstract}
\vspace{-4mm}
\section{Introduction}
\label{intro}

One of the most popular optimization algorithms for training neural networks is \adam~\cite{kingma2014adam}, which is a variant of the general class of \adagrad-type algorithms~\cite{duchi2011adaptive}.
The main novelty of \adam is to apply an exponential moving average (EMA) to gradient estimate (first-order) and to element-wise square-of-gradients (second-order), with parameters $\beta_1$ and $\beta_2$, respectively.

In practice, constant $\beta_1$ and $\beta_2$ values are used (the default
parameters in \textsc{PyTorch} and \textsc{Tensorflow}, for example, are
$\b_1=0.9$ and $\b_2=0.999$).
However, the regret analysis in~\citet{kingma2014adam} requires $\beta_1 \to 0$ with a linear rate, causing a clear discrepancy between theory and practice.

Recently, it has been shown by \citet{reddi2018convergence} that the analysis of
\adam contained a technical issue.
After this discovery, many variants of \adam were proposed with optimal regret guarantees~\cite{reddi2018convergence,chen2018closing,huang2019nostalgic}.
Unfortunately, in all these analyses, the requirement of $\beta_1 \to 0$ is inherited and is needed to derive the optimal $\mathcal{O}(\sqrt{T})$ regret.
In contrast,  methods that are shown to exhibit favorable practical performance continued to use a constant $\beta_1$ in the experiments.

One can wonder whether there is an inherent obstacle~--- in the proposed methods or the setting~--- which prohibits optimal regret bounds with a
constant $\beta_1$.

In this work, we show that this specific discrepancy between the
theory and practice is indeed an artifact of the previous analyses.
We point out the shortcomings responsible for this artifact, and then introduce a new analysis
framework that attains optimal regret bounds with constant $\beta_1$
at no additional cost (and even comes with better constants in the obtained bounds).

\textbf{Our contributions.}
In the convex setting, our technique obtains data-dependent $\mathcal{O}\big( \sqrt{T} \big)$ regret bounds for \amsgrad and \adamnc \cite{reddi2018convergence}.
Moreover, our technique can also be applied to a strongly convex variant of~\adamnc, known as~\sadam~\cite{wang2020sadam}, yielding again data-dependent logarithmic regret with constant $\beta_1$.
To the best of our knowledge, these are the first optimal regret bounds with constant $\beta_1$.

Finally, we illustrate the flexibility of our framework by applying it to zeroth-order (bandit) and nonconvex optimization.
In the zeroth-order optimization setting, we improve on the current best result which requires $\beta_1 \sim  \frac{1}{t}$, and show that a constant $\beta_1$ again suffices.
In the non-convex setting, we recover the existing results in the literature, with a simpler proof and slight improvements in the bounds.

\subsection{Problem Setup}
In online optimization, a loss function $f_t\colon \mathcal{X} \to \mathbb{R}$ is revealed, after a decision vector $x_t\in\mathcal{X}$ is picked by the algorithm. 
We then minimize the regret defined as
\begin{equation}\label{eq:regret}
    R(T)=\sum_{t=1}^T f_t(x_t)-\min_{x\in \cX}\sum_{t=1}^Tf_t(x).
\end{equation}
Our assumptions are summarized below which are the same as in~\cite{reddi2018convergence}.
\begin{assumption}\label{as:1}
~\\
$\triangleright$ $\cX\subset \R^d$ is a compact convex set. \\
$\triangleright$ $f_t\colon \cX\to \R $ is a convex lsc function, $g_t\in
\partial f_t(x_t)$. \\
$\triangleright$ $D=\max\limits_{x,y\in \cX}\n{x-y}_{\infty}$, $G=\max\limits_{t}\n{g_t}_{\infty}$.
\end{assumption}

\subsection{Preliminaries}
We work in Euclidean space $\R^d$ with inner
product $\lr{\cdot,\cdot}$. For vectors $a,b\in \R^d$ all standard
operations $ab$, $a^2$, $a/b$, $a^{1/2}$, $1/a$, $\max\{a,b\}$ are supposed to be
coordinate-wise. For a given vector $a_t\in\mathbb{R}^d$, we denote its $i^{\text{th}}$ coordinate by $a_{t, i}$. 
We denote the vector of all-ones as $\mathbf{1}$.
We use $\diag(a)$ to denote a $d\times d$ matrix
which has $a$ in its diagonal, and the rest of its elements are $0$.  For $v_i > 0, \forall i =1,\ldots,d$, we
define a weighted norm
\begin{equation*}
    \|x\|_v^2 := \langle x, (\diag{v}) x \rangle
\end{equation*}
and a weighted projection operator onto $\cX$
\begin{equation}
 P^v_{\mathcal{X}}(x) = \arg \min_{y\in\mathcal{X}} \| y- x \|^2_{v}.\label{eq: w_proj_def}
\end{equation}
We note that $\forall x, y \in \mathbb{R}^d$, $P_{\cX}^v$ is
nonexpansive, that is 
\begin{equation}
\| P^v_{\mathcal{X}} (y) - P^v_{\mathcal{X}} (x) \|_v \leq \| y - x \|_v.\label{eq: nonexp}
\end{equation}

\section{Related work}
\subsection{Convex world}
In the setting of online convex optimization (OCO), Assumption~\ref{as:1} is
standard~\cite{hazan2016introduction,duchi2011adaptive}. It allows us
to consider nonsmooth stochastic minimization (though we are not
limited to this setting), and even allows for adversarial loss functions.

The algorithms \amsgrad and \adamnc were proposed
by~\citet{reddi2018convergence} to fix the issue in the original proof
of \adam~\cite{kingma2014adam}.  However, as the proof template
of~\citet{reddi2018convergence} follows very closely the proof
of~\citet{kingma2014adam}, the requirement for $\beta_1 \to 0$ remains
in all the regret guarantees of these algorithms.  In particular, as noted
by~\citet[Corollary 1, 2]{reddi2018convergence}, a schedule
of $\beta_{1t} = \beta_1 \lambda^{t-1}$ is needed for obtaining
optimal regret.  ~\citet{reddi2018convergence} also noted that regret
bounds of the same order can be obtained by setting
$\beta_{1t} = \beta_1/t$.  On the other hand, in the numerical
experiments, a constant value $\beta_{1t} = \beta_1$ is used
consistent with the huge literature following~\citet{kingma2014adam}.

Following~\citet{reddi2018convergence}, there has been a surge of interest in proposing new variants of~\adam with good practical properties;
to name a few, \textsc{Padam} by~\citet{chen2018closing}, \textsc{Adabound} and \textsc{Amsbound} by~\citet{luo2018adaptive,savarese2019convergence}, \textsc{Nostalgic Adam} by~\citet{huang2019nostalgic}.
As the regret analyses of these methods follow very closely the analysis of~\citet{reddi2018convergence}, the resulting bounds inherited the same shortcomings.
In particular, in all these algorithms, to achieve $\mathcal{O}\big(\sqrt{T}\big)$ regret, one needs either $\beta_{1t} = \beta_1 \lambda ^{t-1}$ or $\beta_{1t} = \frac{\beta_1}{t}$.
On the other hand, the experimental results reported on these algorithms note that a constant value of $\beta_1$ is used in practice in order to obtain better performance.

Similar issues are present in other problem settings.  For strongly
convex optimization,~\citet{wang2020sadam} proposed the~\sadam
algorithm as a variant of~\adamnc, which exploits strong
convexity to obtain $\mathcal{O}\big(\log T\big)$ regret.
\sadam was shown to exhibit favorable practical performance
in the experimental results of~\citet{wang2020sadam}.  However, the
same discrepancy exists as with previous~\adam variants: a
linearly decreasing $\beta_{1t}$ schedule is required in theory but a constant $\beta_{1t}=\beta_1$ is used in practice.

One work  that tried to address this issue is that of~\citet{fang2019convergence}, where the authors focused on OCO with strongly convex loss functions and derived an $\mathcal{O}(\sqrt{T})$ regret bound with a constant value of $\beta_1 \leq \frac{\mu\alpha}{1+\mu\alpha}$, where $\mu$ is the strong convexity constant and $\alpha$ is the step size that is set as $\alpha_1 / \sqrt{T}.$~\citep[Theorem 2]{fang2019convergence}.
However, this result is still not satisfactory, since the obtained bound for $\beta_1$ is weak: both strong
convexity $\mu$ and the step size $\frac{\alpha_1}{\sqrt{T}}$ are
small. This does not allow for the standard choices of $\beta_1 \in (0.9, 0.99)$.

Moreover, a quick look into the proof of~\citet[Theorem 2]{fang2019convergence} reveals that the proof in fact follows the same lines as~\citet{reddi2018convergence} with the difference of  using the contribution of strong convexity to get rid of the spurious terms that require $\beta_1 \to 0$.
Therefore, it is not surprising that the theoretical bound for $\beta_1$ depends on $\mu$ and $\alpha$ and can only take values close to $0$.
Second, in addition to the standard Assumption~\ref{as:1},~\citet{fang2019convergence} also assumes strong convexity, which is a quite stringent assumption by itself.
In contrast, our approach does not follow the
lines of~\citet{reddi2018convergence}, but is an alternative way that does not encounter the same roadblocks.

\subsection{Nonconvex world}
A related direction to what we have reviewed in the previous
subsection is to analyze~\adam-type algorithms without convexity
assumptions.  When convexity is removed, the standard
setting in which the algorithms are analyzed, is stochastic
optimization with a smooth loss function and no
constraints~\cite{chen2019convergence,zhou2018convergence,zou2019sufficient}.
As a result, these algorithms, compared to the convex counterparts, do
not perform projections in the update step of $x_{t+1}$~(\textit{cf.},
Algorithm~\ref{alg:amsgrad}).  

In addition to smoothness, bounded
gradients are assumed, which is also restrictive, as many nonconvex
functions do not satisfy this property. Indeed, one can show that it
is equivalent to the Lipschitz continuity of the function (not its
gradient!).~Under these assumptions, the standard results bound the
minimum gradient norm across all iterations.

An interesting phenomenon in this line of work is that a constant
$\beta_1 < 1$ is permitted for the theoretical results, which may seem
like weakening our claims.  However, it is worth noting that these
results do not imply any guarantee for regret in OCO setting.  

Indeed,
adding the convexity assumption to the setting of unconstrained, smooth
stochastic optimization, would only help obtaining a gradient norm
bound in the averaged iterate, rather than the minimum across all
iterations.  However, this bound does not imply any guarantee in the
objective value, unless more stringent Polyak-Lojasiewicz or strong convexity requirements are added in the mix.  

Moreover, in the OCO
setting that we analyze, loss functions are nonsmooth, and there
exists a constraint onto which a projection is performed in the
$x_{t+1}$ step~(\textit{cf.},  Algorithm~\ref{alg:amsgrad}).  Finally,
online optimization includes stochastic optimization as a special case. Given the difference of assumptions, the
analyses
in~\cite{chen2019convergence,zhou2018convergence,zou2019sufficient}
indeed do not help obtaining any regret guarantee for standard OCO.

A good example demonstrating this difference on the set of assumptions is the work~\cite{chen2019zo}.
In this paper, a variant of~\amsgrad is proposed for zeroth order optimization and it is analyzed in the convex and nonconvex settings.
Consistent with the previous literature in both, convergence result for the nonconvex setting allows a constant $\beta_1 < 1$~\citep[Theorem 1]{chen2019zo}.
However, the result in the convex setting requires a decreasing schedule such that $\beta_{1t} = \frac{\beta_1}{t}$~\citep[Proposition 4]{chen2019zo}.

As we highlighted above, the analyses in  convex/nonconvex
settings follow different paths and the results or techniques are not transferrable to each other.
Thus, our main aim in this paper is to bridge the gap in the understanding of regret analysis for OCO and propose a new analytic framework.
As we see in the sequel, our analysis not only gives the first results in OCO setting, it is also general enough to apply to the abovementioned nonconvex optimization case and recover similar results as the existing ones.
\vspace{-3mm}
\section{Main results}\vspace{-2mm}
\subsection{Dissection of the standard analysis}\label{eq: prev_analysis}
We start by describing the shortcoming of the previous approaches
in~\cite{reddi2018convergence, wang2020sadam} and, then explain the
mechanism that allows us to obtain regret bounds with constant
$\beta_1$. In this subsection, for full generality,
we assume that the update for $m_t$ is not done with $\b_1$, but with $\b_{1t}$, as in~\cite{reddi2018convergence,kingma2014adam}:
\begin{equation}\label{eq:mt}
m_t =\b_{1t}m_{t-1} + (1-\b_{1t})g_{t}.
\end{equation}
The standard way to analyze Adam-type algorithms is to
start by the nonexpansiveness property~\eqref{eq: nonexp} and to write
\begin{align*}
\| x_{t+1} - x\|^2_{\hat{v}_t^{1/2}} \leq
\| x_t - x \|^2_{\hat{v}_t^{1/2}} &- 2\a_t\langle  m_t, x_t - x \rangle \notag \\
&+ \alpha_t^2 \| m_{t} \|^2_{\hat{v}_{t}^{-1/2}}.\notag
\end{align*}
Then using~\eqref{eq:mt}, one can deduce
\begin{multline}
(1-\beta_{1t})\langle g_t, x_t - x \rangle \leq-\beta_{1t} \langle m_{t-1}, x_{t} - x \rangle +\frac{\alpha_t}{2} \| m_t \|^2_{\hat{v}_t^{-1/2}}\notag \\
+ \frac{1}{2\alpha_t}\left( \| x_t - x\|^2_{\hat{v}_{t}^{1/2}} - \| x_{t+1} - x\|^2_{\hat{v}_{t}^{1/2}}\right).\notag
\end{multline}
Let us analyze the above inequality. Its left-hand side is exactly
what we want to bound, since by convexity $R(T)\leq
\sum_{t=1}^T\lr{g_t,x_t-x}$. The last two terms in the right-hand side
are easy to analyze, all of them can be bounded in a standard way
using just definitions of $\hv_t$, $m_t$, and $\a_t$.

What can we do with the term $-\b_{1t} \lr{m_{t-1},x_t-x}$? Analysis
in~\cite{reddi2018convergence} bounds it with Young's inequality
\begin{multline*}
-\beta_{1t} \langle m_{t-1}, x_{t} - x \rangle \leq \frac{\beta_{1t}}{2\alpha_t} \| x_t-x \|^2_{\hat{v}_t^{1/2}} \\
+ \frac{\beta_{1t}\alpha_t}{2} \| m_{t-1} \|^2_{\hat{v}_t^{-1/2}}.
\end{multline*}
The term $\frac{\beta_{1t}}{2\alpha_t} \| x_t-x \|^2_{\hat{v}_t^{1/2}}$ is precisely what leads to the
second term in the regret bound in~\citep[Theorem
4]{reddi2018convergence}. Since $\alpha_t = \frac{\alpha}{\sqrt{t}}$, one
must  require $\beta_{1t} \to 0$.

Note that the update for $x_{t+1}$ has a projection. This is important, since
otherwise a solution must lie in the interior of $\cX$, which is not
the case in general for problems with a compact domain. However, let
us assume for a moment that the update for $x_{t+1}$ does not have any
projection.
In this simplified setting, applying the following trick will
work.

Recall that
$x_t = x_{t-1} - \alpha_{t-1}\hat{v}_{t-1}^{-1/2}m_{t-1}$,
or equivalently
$m_{t-1} = \frac{1}{\alpha_{t-1}}\hat{v}_{t-1}^{1/2}  (x_{t-1} - x_t)$.
Plugging it into the error term $\langle m_{t-1}, x_{t} - x \rangle$ yields
\begin{align*}
&-\lr{m_{t-1},x_{t}-x}=- \frac{1}{\alpha_{t-1}} \langle \hat{v}_{t-1}^{1/2}  (x_{t-1} - x_t), x_t - x \rangle \\ 
 &= \frac{1}{2\alpha_{t-1}} \bigg[ \|x_t - x_{t-1} \|^2_{\hat{v}_{t-1}^{1/2}} + \| x_t -x \|^2_{\hat{v}^{1/2}_{t-1}} \\
 &\qquad\qquad\qquad\qquad\qquad\quad~~ - \| x_{t-1} -x \|^2_{\hat{v}^{1/2}_{t-1}} \bigg]\\
 &\leq \frac{1}{2}\alpha_{t-1}\|m_{t-1} \|^2_{\hat{v}_{t-1}^{-1/2}} + \frac{1}{2}\| x_t -x \|^2_{\hat{v}^{1/2}_{t}/\alpha_t} \\
 &\qquad\qquad\qquad\qquad\quad- \frac{1}{2}\| x_{t-1} -x \|^2_{\hat{v}^{1/2}_{t-1}/\alpha_{t-1}},
\end{align*}
where the second equality follows from the Cosine Law and the first inequality is from
$x_t - x_{t-1} = -\alpha_{t-1}\hat{v}_{t-1}^{-1/2}m_{t-1}$ and
$\hat{v}_{t}^{1/2}/\alpha_t \geq \hat{v}_{t-1}^{1/2}
/\alpha_{t-1}$. We now compare this bound with the previous one. The term
$\alpha_{t-1}\|m_{t-1} \|^2_{\hat{v}_{t-1}^{-1/2}}$, as we already
observed, is good for summation (\textit{cf.},~\Cref{lem: sum_grad_norm}). And other two terms are going to cancel after summation over $t$. 
Hence, it is easy to finish the analysis to conclude
$\cO(\sqrt T)$ regret with a fixed $\b_{1t}=\b_1$.

Unfortunately, the update for $x_{t+1}$ does have a projection,
without it the assumption for the domain to be bounded is very
restrictive. This prevents us from using the above trick.  Its
message, however, is that it is feasible to expect a good bound even
with a fixed $\b_{1t}$, and under the same assumptions on the problem
setting.

For having a more general technique to handle $\beta_1$, we will take a
different route in the very beginning~---~we will analyze the term
$\lr{g_t, x_t-x}$ in a completely different way, without resorting to
any crude inequality as in~\cite{reddi2018convergence}. Basically,
this idea can be applied to any framework with a similar update for
the moment $m_t$.

\subsection{A key lemma}

As we understood above, the presence of the projection complicates
handling $\lr{m_{t-1},x_t-x}$.  A high level explanation for the cause
of the issue is that the standard analysis does not leave much
flexibility, since it uses nonexpansiveness in the very beginning.
\begin{lemma}\label{lem: decop}
Under the definition
\begin{equation*}
m_t = \beta_1 m_{t-1} + (1-\beta_1) g_t,
\end{equation*}
it follows that
\begin{multline*}
\langle g_t, x_t - x \rangle = \langle m_{t-1}, x_{t-1} - x \rangle \\-\frac{\beta_1}{1-\beta_1} \langle m_{t-1}, x_t - x_{t-1} \rangle \\
+\frac{1}{1-\beta_1} \left( \langle m_t, x_t - x\rangle - \langle m_{t-1}, x_{t-1} - x\rangle \right).
\end{multline*}

\end{lemma}
The main message of Lemma~\ref{lem: decop} is that the decomposition of $m_t$, in the second part of the analysis in Section~\ref{eq: prev_analysis} is now done before using nonexpansiveness, therefore there would be no need for using Young's inequality which is the main shortcoming of the previous analysis.

Upon inspection on the bound, it is now easy to see that the last two terms will telescope.
The second term can be shown to be of the order $\alpha_t \| m_t \|^2_{\hat v_t^{-1/2}}$, and as we have seen before, summing this term will give $\mathcal{O}\big(\sqrt T \big)$.
To see that the first term is also benign, a high level explanation is to notice that $m_{t-1}$ is the gradient estimate used in the update $x_{t} = P_{\mathcal{X}}^{\hat v_{t-1}^{1/2}} (x_{t-1} - \alpha_{t-1} \hat{v}_{t-1}^{-1/2} m_{t-1})$, therefore it can be analyzed in the classical way.
We will now proceed to illustrate the flexibility of the new analysis on three most popular \adam variants that are proven to converge.

\subsection{\amsgrad}
\amsgrad\ is proposed by~\cite{reddi2018convergence} as a fix to \adam.
The algorithm incorporates an extra step to enforce monotonicity of second moment estimator $\hat{v}_t$.
\begin{algorithm}[h]
\caption{\amsgrad \cite{reddi2018convergence}}
\label{alg:amsgrad}
\begin{algorithmic}[1]
    \STATE {\bfseries Input:} $x_1 \in \cX$,
    $\a_t=\frac{\alpha}{\sqrt{t}}$, $\a>0$, $\beta_1 < 1$, $\beta_2 < 1$,\\
    $m_0=v_0=0$, $\hv_0=\epsilon\mathbf{1}$, $\epsilon \geq 0$
    \FOR{$t = 1,2\ldots $} 
	\STATE $g_t \in  \partial f_t(x_t)$
         \STATE $m_{t}= \beta_{1}m_{t-1} + (1-\beta_{1})g_t$
        \STATE $v_t= \beta_{2} v_{t-1} + (1-\beta_{2}) g_t^2$
         \STATE ${\hat{v}_t= \max(\hat{v}_{t-1}, v_t) }$
         \STATE $x_{t+1}= P^{{\hat{v}_t}^{1/2}}_{\mathcal{X}} (x_t - \alpha_t {\hat{v}_t}^{-1/2} m_t )$
	\ENDFOR
\end{algorithmic}	
\end{algorithm}

The regret bound for this algorithm in~\citep[Theorem 4, Corollary 1]{reddi2018convergence} requires a decreasing $\beta_1$ at least at the order of $1/t$ to obtain $\mathcal{O}\big(\sqrt T \big)$ worst case regret.
Moreover, it is easy to see that a constant $\beta_1$ results in $\mathcal{O}\big( T\sqrt{T} \big)$ worst case regret in~\citep[Theorem 4]{reddi2018convergence}.

We now present the following theorem which shows that the same $\mathcal{O}\big(\sqrt T \big)$ can be obtained by \amsgrad\ under the same structural assumptions as~\cite{reddi2018convergence}.

\begin{theorem}\label{th:amsgrad}
Under Assumption~\ref{as:1}, $\beta_1 < 1$, $\beta_2 < 1$, $\gamma = \frac{\beta_1^2}{\beta_2} < 1$, and $\epsilon > 0$, \amsgrad achieves the regret
\begin{multline}\label{eq:regret_amsgrad}
  R(T)  \leq \frac{D^2\sqrt{T}}{2\a(1-\b_1)}\sum_{i=1}^d \hv^{1/2}_{T,i}\\ +\frac{\alpha\sqrt{1+\log T}}{\sqrt{(1-\beta_2)(1-\gamma)}} \sum_{i=1}^d
  \sqrt{\sum_{t=1}^T g_{t,i}^2}.
\end{multline}
\end{theorem}
%\begin{remark}
We would like to note that our bound for $R(T)$ is also better than
the one in~\cite{reddi2018convergence} in term of constants. We have
only two terms in contrast to three in~\cite{reddi2018convergence} and
each of them is strictly smaller than their counterparts in
\cite{reddi2018convergence}.  The reason is that we used i)~new way of
decomposition $\lr{g_t,x_t-x}$ as in~\Cref{lem: decop}, ii)~wider
admissible range for $\b_1,\b_2$, iii) more refined estimates for
analyzing terms. For example, the standard analysis to estimate
$\n{m_t}^2_{\hv_t^{-1/2}}$ uses several Cauchy-Schwarz inequalities. We
instead give a better bound by applying generalized H\"older
inequality~\cite{beckenbachinequalities}.
%\end{remark}

Another observation is that having a constant $\beta_1$ explicitly improves the last term in the regret bound.
If one uses a non-decreasing $\beta_1$, instead of constant $\beta_1$, then this term will have an additional multiple of $\frac{1}{(1-\beta_1)^2}$.
Given that in general one chooses $\beta_1$ close to $1$, this factor is significant.

\begin{remark}
    Notice that~\Cref{th:amsgrad} requires $\epsilon > 0$ in order to
    have the weighted projection operator in~\eqref{eq: w_proj_def}
    well-defined. Such a requirement is common in the literature for
    theoretical analysis, see~\citep[Theorem 5]{duchi2011adaptive}.
    In practice, however, one can set $\epsilon = 0$.
\end{remark}

\begin{sproof}
    We sum $\lr{g_t,x_t-x}$ from  Lemma~\ref{lem: decop}  over $t$, use $m_0 = 0$ to get
\begin{multline*}
\sum_{t=1}^T \langle g_t, x_t - x \rangle \leq \underbrace{\sum_{t=1}^T \langle m_{t}, x_{t} - x \rangle}_{S_1}  \\
+\frac{\beta_1 }{1-\beta_1}\underbrace{\sum_{t=1}^T\langle m_{t-1}, x_{t-1} - x_t \rangle}_{S_2} \\
+\frac{\beta_1}{1-\beta_1} \underbrace{\langle m_T, x_T - x\rangle}_{S_3}.
\end{multline*}
By H\"older inequality, we can show that
\begin{equation*}
S_2 \leq \sum_{t=1}^{T-1}\alpha_{t} \|m_{t}\|_{\hat{v}_{t}^{-1/2}}^2.
\end{equation*}
By using the fact that $\hat{v}_{t, i} \geq \hat{v}_{t-1, i}$, and the same estimation as deriving $S_2$,
\begin{align*}
S_1 \leq \frac{D^2}{2\alpha_T} \sum_{i=1}^d \hat{v}_{T, i}^{1/2} + \sum_{t=1}^T \frac{\alpha_t}{2} \|m_t \|^2_{\hat v_t^{-1/2}}.
\end{align*}
By H\"older and Young's inequalities, we can bound $S_3$ as
\begin{equation}
S_3 \leq \a_T\n{m_T}_{\hv_T^{-1/2}}^2+
\frac{D^2}{4\a_T} \sum_{i=1}^d \hv^{1/2}_{T,i}.\notag
\end{equation}
Lastly, we see that $\alpha_t \| m_t \|^2_{\hat v_t^{-1/2}}$ is common
in all these terms and it is well known that this term is good for summation
\begin{equation}
\sum_{t=1}^T \alpha_t \| m_t \|^2_{\hat v_t^{-1/2}} \leq \frac{(1-\beta_1)\alpha\sqrt{1+\log T}}{\sqrt{(1-\beta_2)(1-\gamma)}}{\displaystyle\sum_{i=1}^d
  (\sum_{t=1}^T g_{t,i}^2})^{\frac{1}{2}}. \notag
\end{equation}
Combining the terms gives the final bound.
\end{sproof}

Finally, if we are interested in the worst case scenario, it is clear
that \Cref{th:amsgrad} gives regret $R(T) =
\cO(\sqrt{\log(T)T})$. 
A quick look into the calculations yields that if one uses the worst case bound $g_{t, i} \leq G$, then the bound will not include a logarithmic term.
However, then the data-dependence of the bound will be lost.
It is not clear if one can obtain a data-dependent $\mathcal{O}(\sqrt{T})$ regret bound.
In the following corollary, we give a partial answer to this question.
\begin{corollary}\label{cor:amsgrad}
Under Assumption~\ref{as:1}, $\beta_1 < 1$, $\beta_2 < 1$, $\gamma = \frac{\beta_1^2}{{\beta_2}} < 1$, and $\epsilon > 0$, \amsgrad achieves the regret
\begin{multline}\label{eq:regret_amsgrad2}
R(T) \leq \frac{D^2\sqrt{T}}{2\a(1-\b_1)}\sum_{i=1}^d \hv^{1/2}_{T,i}\\
+  \frac{\alpha\sqrt{G}}{\sqrt{1-\beta_2}(1-\gamma)} \sum_{i=1}^d \sqrt{\sum_{t=1}^T |g_{t,i}|}.
\end{multline}
\end{corollary}
We remark that even though this bound does not contain a $\log(T)$ term, thus better in the worst-case, its data-dependence is actually worse than the standard bound.
Standard bound contains $g_{t, i}^2$ whereas bound above contains $\vert g_{t, i} \vert$.
Therefore, when the values $g_{t, i}$ are very small, the bound with $\log{T}$ can be better.
We leave it as an open question to have a $\sqrt{T}$ bound with the same data-dependence as the original bound.
\subsection{\adamnc}
Another variant that is proposed by~\citet{reddi2018convergence} as a
fix to \adam\ is \adamnc\ which features an increasing schedule for
$\beta_{2t}$. In particular, one sets $\b_{2t}=1-\frac{1}{t}$ in
\begin{equation}
v_t = \beta_{2t} v_{t-1} + (1-\beta_{2t})g_t^2,\notag
\end{equation}
 that results
in the following expression for $v_t$
\begin{equation}
v_t = \frac{1}{t} \sum_{j=1}^t g_j^2,\notag
\end{equation}
which is a reminiscent of \adagrad~\cite{duchi2011adaptive}. In fact,
to ensure that $P_{\cX}^{v_t^{1/2}}$ is well-defined, one needs to
consider the more general update $v_t = \frac{1}{t} \left( \epsilon \mathbf{1}+ \sum_{j=1}^t
g_j^2\right)$ similar to the previous case with~\amsgrad.

\begin{algorithm}[h]
\caption{\adamnc \cite{reddi2018convergence}}
\label{alg:adamnc}
\begin{algorithmic}[1]
    \STATE {\bfseries Input:} $x_1 \in \cX$,
    $\a_t=\frac{\alpha}{\sqrt{t}}$, $\a>0$, $\beta_1 < 1$, $\epsilon\geq 0$,
    $m_0=0$.
    \FOR{$t = 1,2\ldots $} 
	\STATE $g_t \in  \partial f_t(x_t)$
         \STATE $m_{t}= \beta_{1}m_{t-1} + (1-\beta_{1})g_t$
        \STATE $v_t= \frac{1}{t} \left(\sum_{j=1}^t g_j^2 + \epsilon\mathbf{1}\right)$
         \STATE $x_{t+1}= P^{{v_t}^{1/2}}_{\mathcal{X}} (x_t - \alpha_t {{v}_t}^{-1/2} m_t )$
	\ENDFOR
\end{algorithmic}	
\end{algorithm}

\adamnc is analyzed in~\citep[Theorem 5, Corollary
2]{reddi2018convergence} and similar to \amsgrad\ it has been shown to
exhibit $\mathcal{O}\big(\sqrt T \big)$ worst case regret only when
$\beta_1$ decreases to $0$. We show in the following theorem that the
same regret can be obtained with a constant $\beta_1$.

\begin{theorem}\label{th: adamnc}
Under Assumption~\ref{as:1}, $\beta_1 < 1$, and $\epsilon > 0$, \adamnc achieves the regret
\begin{align}
R(T) &\leq \frac{D^2\sqrt{T}}{2\alpha(1-\beta_1)} \sum_{i=1}^d v_{T, i}^{1/2}+ \frac{2\alpha}{1-\beta_1}\sum_{i=1}^d \sqrt{\sum_{t=1}^T g_{t, i}^2}.\notag
\end{align}
\end{theorem}

We skip the proof sketch of this theorem as it will have the same steps as~\amsgrad, just different estimation for $\alpha_t \| m_t \|^2_{v_t^{-1/2}}$, due to different $v_t$.

The full proof is given in the appendix.

Compared with the bound from~\citep[Corollary 2]{reddi2018convergence}, we see again that constant $\beta_1$ not only removes the middle term of~\citep[Corollary 2]{reddi2018convergence} but improves the last term of the bound by a factor of $(1-\beta_1)^2$.

\subsection{\sadam}
It is known that \adagrad~can obtain logarithmic regret~\cite{duchi2010adaptive_techreport}, when the loss functions satisfy $\mu$-strong convexity, defined as
\begin{equation*}
f(x) \geq f(y) + \langle g, x - y \rangle + \frac{\mu}{2} \| y-x \|^2,
\end{equation*}
$\forall x, y \in \mathcal{X}$ and $g \in \partial f(y)$.

A variant of \adamnc\ for this setting is proposed in~\citep[Theorem 1]{wang2020sadam} and shown to obtain logarithmic regret, only with the assumption that $\beta_1$ decreases linearly to $0$.

\begin{algorithm}[h]
\caption{\sadam~\cite{wang2020sadam}}
\label{alg:sadam} 
\begin{algorithmic}[1]
    \STATE {\bfseries Input:} $x_1 \in \cX$,
    $\a_t=\frac{\alpha}{{t}}$, $\a>0$, $\beta_1 < 1$, $m_0=0$, $\epsilon\geq 0$, $\beta_{2t} = 1-1/t$.
    \FOR{$t = 1,2\ldots $} 
	\STATE $g_t \in  \partial f_t(x_t)$
         \STATE $m_{t}= \beta_{1}m_{t-1} + (1-\beta_{1})g_t$
        \STATE $v_t= \beta_{2t} v_{t-1} + (1-\beta_{2t}) g_t^2$
        \STATE $\hat{v}_t = v_t + \frac{\epsilon\mathbf{1}}{t}$
         \STATE $x_{t+1}= P^{{\hat{v}_t}}_{\mathcal{X}} (x_t - \alpha_t {\hat{v}_t}^{-1} m_t )$
	\ENDFOR
\end{algorithmic}	
\end{algorithm}

Similar to \amsgrad\ and \adamnc, our new technique applies to \sadam
to show logarithmic regret with a constant $\beta_1$ under the same assumptions as~\cite{wang2020sadam}.

\begin{theorem}\label{th: sadam}
Let Assumption~\ref{as:1} hold and $f_t$ be $\mu$-strongly
convex, $\forall t$. Then, if $\beta_1 < 1$, $\epsilon > 0$, and $\alpha \geq \frac{G^2}{\mu}$, \sadam achieves
\begin{align}
R(T) \leq \frac{\beta_1 d G D}{1-\beta_1}  + \frac{\alpha}{1-\beta_1} \sum_{i=1}^d \log\left( \frac{\sum_{t=1}^T g_{t, i}^2}{\epsilon}+1 \right).\notag
\end{align}
\end{theorem}
%\begin{remark}
Consistent with the standard literature of OGD~\cite{hazan2007logarithmic}, to obtain the logarithmic regret,  first step size $\alpha$ has a lower bound that depends on strong convexity constant $\mu$.
Compared with the requirement of~\cite{wang2020sadam} for $\alpha \geq \frac{G^2}{\mu(1-\beta_1)}$, our requirement is strictly milder as $1-\beta_1 \leq 1$ and in practice since $\beta_1$ is near $1$, it is much milder.
We also remark that our bound is again strictly better than~\cite{wang2020sadam}.
Consistent with our previous results, we remove a factor of $\frac{1}{(1-\beta_1)^2}$ from the last term of the bound, compared to~\citep[Theorem 1]{wang2020sadam}.

We include the proof sketch to highlight how strong convexity helps in the analysis.
%\end{remark}
\begin{sproof}
We will start the same as proof sketch of Theorem~\ref{th:amsgrad} to get
\begin{equation}
\sum_{t=1}^T \langle g_t, x_t - x \rangle \leq S_1 + \frac{\beta_1}{1-\beta_1} S_2 + \frac{\beta_1}{1-\beta_1}S_3,\notag
\end{equation}
with the definitions of $S_1$, $S_2$, $S_3$ from the proof sketch of Theorem~\ref{th:amsgrad}.

Now, due to strong convexity, one gets an improved estimate for the left-hand side,
\begin{equation}
\langle g_t, x_t - x\rangle \geq f_t(x_t) - f_t(x) + \frac{\mu}{2} \| x_t - x \|^2,\notag
\end{equation}
resulting in
\begin{multline}
R(T) \leq S_1 + \frac{\beta_1}{1-\beta_1} S_2 + \frac{\beta_1}{1-\beta_1}S_3 \\
-\sum_{t=1}^T \frac{\mu}{2} \| x_t - x \|^2.\label{eq: str}
\end{multline}
Similar as before, we note the bound for $S_2$ as
\begin{align}
S_2 \leq \sum_{t=1}^{T-1} \alpha_t \| m_t \|^2_{\hat v_t^{-1/2}}.
\end{align}

For $S_1$, one does not finish the estimation as before, but keep some terms that will be gotten rid of using strong convexity, and use the same estimation as $S_2$ to obtain
\begin{multline}
S_1 \leq \sum_{t=1}^T \sum_{i=1}^d \left( \frac{\hat{v}_{t, i}}{2\alpha_t} - \frac{\hat{v}_{t-1, i}}{2\alpha_{t-1}} \right) (x_{t, i}-x_i)^2 \\
+ \sum_{t=1}^T \frac{\alpha_t}{2} \|m_t \|^2_{\hat v_t^{-1/2}}.\label{eq: s11}
\end{multline}

As strong convexity gives more flexibility in the analysis, one can select $\alpha_t = \frac{\alpha}{t}$, resulting in an improved bound
\begin{equation}
\sum_{t=1}^T \frac{\alpha_t}{2} \|m_t \|^2_{\hat v_t^{-1/2}} \leq \alpha \sum_{i=1}^d \log\left( \frac{\sum_{t=1}^T g_{t, i}^2}{\epsilon}+1 \right).\label{eq: sadam_imp}
\end{equation}

It is now easy to see that the negative term in~\eqref{eq: str}, when first step size $\alpha$ is selected properly, can be used to remove the first term in the bound of~\eqref{eq: s11}.

It only remains to use H\"older inequality on $S_3$, combine the estimates and use~\eqref{eq: sadam_imp} to get the final bound.
\end{sproof}

\section{Extensions}
In this section, we further demonstrate the applicability of our analytic framework in different settings.
First, we focus on the recently proposed zeroth-order version of \amsgrad which required decreasing $\beta_1$ in the convex case~\citep[Proposition 4]{chen2019zo}, and we show that the same guarantees can be obtained with constant $\beta_1$.
Second, we show how to recover the known guarantees in the nonconvex setting, with small improvements.
Finally, we extend our analysis to show that it allows any non-increasing variable $\beta_{1t}$ schedule.
\subsection{Zeroth order \adam}
We first recall the setting of~\cite{chen2019zo}, where a zeroth order variant of \amsgrad is proposed.
The problem is
\begin{equation}
x_\star \in \arg\min_{x\in\mathcal{X}} f(x) := \mathbb{E}_\xi\left[ f(x; \xi) \right].
\end{equation}
We note that this stochastic optimization setting corresponds to a
special case of general OCO, with independent and identically
distributed loss functions $f(x;\xi)$, indexed by  $\xi$.

The algorithm ZO-AdaMM~\cite{chen2019zo} is similar to \amsgrad applied with a zeroth order gradient estimator $\hat{g}_t$, instead of regular gradient $g_t$.
The gradient estimator is computed by
\begin{equation}\label{eq: zero_grad_def}
\hat{g}_t = (d/\mu)\left[ f(x_t+\mu u; \xi_t) - f(x_t; \xi_t)\right]u,
\end{equation}
where $\xi_t$ is the sample selected at iteration $t$, $u$ is a random vector drawn with uniform distribution from the sphere of a unit ball and $\mu$ is a sampling radius \textendash\ or smoothing \textendash\ parameter.

The benefit of this gradient estimator is that it is an unbiased estimator of the randomized smoothed version of $f$, i.e.,
\begin{equation}
f_{\mu}(x) = \mathbb{E}_{u\sim U_B}\left[ f(x+\mu u) \right].
\end{equation}
From standard results in the zeroth-order optimization literature, it follows that $\mathbb{E}_u \left[ \hat{g}_t \right] = \nabla f_\mu(x_t, \xi_t) = \nabla f_{t, \mu}(x_t)$.
Moreover, for  $L_c$-Lipschitz $f$ and any $x\in\mathcal{X}$, we also have $\| f_\mu(x) -
f(x) \| \leq \mu L_c$.

Two cases are analyzed by~\citet{chen2019zo}: convex $f$ and nonconvex $f$.
The authors proved guarantees with constant $\beta_1$ for nonconvex $f$~\citep[Proposition 2]{chen2019zo}.
However, surprisingly, their result for convex $f$ requires $\beta_{1t} = \frac{\beta_1}{t}$~\citep[Proposition 4]{chen2019zo}.

We identify that this discrepancy is due to the fact that their proof follows the same path as the standard regret analysis of~\citet{reddi2018convergence}.
We give below a simple corollary of our technique showing that the same guarantees for convex $f$ can be obtained with constant $\beta_1$.
\begin{proposition}\label{prop: zero}
Assume that $f$ is convex, $L$-smooth, and $L_c$-Lipschitz, $\mathcal{X}$ is compact with diameter $D$.
Then ZO-AdaMM with $\beta_1, \beta_2 < 1$, $\gamma =\frac{
\beta_1^2}{\beta_2} < 1$ achieves
\begin{multline*}
\mathbb{E}\left[ \sum_{t=1}^T f_{t, \mu}(x_t) - f_{t, \mu}(x_\star) \right] \leq \frac{D^2\sqrt{T}}{2\a(1-\b_1)}\sum_{i=1}^d \mathbb{E}\left[ \hv^{1/2}_{T,i}\right]  \\
+\frac{\alpha\sqrt{1+\log T}}{\sqrt{(1-\beta_2)(1-\gamma)}} \sum_{i=1}^d
    \sqrt{\sum_{t=1}^T \mathbb{E}\left[ \hat{g}_{t,i}^2\right]}.
\end{multline*}
\end{proposition}
To finish the arguments, one can use standard bounds in zeroth order optimization, as in~\citet{chen2019zo}.
Compared with~\citet[Proposition 4]{chen2019zo}, the same remarks hold as for~\amsgrad.
Not only our result allows constant $\beta_1$, but it also comes with better constants.

\subsection{Nonconvex \amsgrad}
In this section, we focus on the nonconvex, unconstrained, smooth, stochastic optimization setting:
\begin{equation}
\min_{x\in\mathbb{R}^d} f(x) := \mathbb{E}_{\xi}[f(x; \xi)].\notag
\end{equation}
More concretely, in this subsection we are working under the following assumption.
\begin{assumption}\label{as:2}
~\\
$\triangleright$ $f\colon \R^d \to \R $ is $L$-smooth, \quad $G=\max\limits_{t}\n{\nabla f(x_t)}_{\infty}$ \\
$\triangleright$ $f_t(x) = f(x, \xi_t)$\\
$\triangleright$ $x_{\star}\in \arg\min_x f(x)$ exists.
\end{assumption}
This is the only setting where theoretical guarantees with constant $\beta_1$ are known in the literature.
We show in this section that our new analysis framework is not restricted to convex case, but it is flexible enough to also cover this case.
We provide an alternative proof to those given in~\cite{chen2019convergence,zhou2018convergence}.
Specifically, both proofs in~\cite{chen2019convergence,zhou2018convergence} exploits the fact that, as $\mathcal{X}=\mathbb{R}^d$, there is no projection step in \amsgrad.
To handle first-order moment, these papers define an auxiliary iterate $z_t = x_t + \frac{\beta_1}{1-\beta_1}(x_t - x_{t-1})$, and invoke smoothness with $z_{t+1}$ and $z_t$.

We give a different and simpler proof using our new analysis, without defining $z_t$. 
In terms of guarantees, we recover the same rates, with slightly better constants.

\begin{theorem}\label{th: nonconvex}

Under \Cref{as:2}, $\beta_1<1$, $\beta_2 < 1$, and $\gamma =
\frac{\beta_1^2}{\beta_2} < 1$ \amsgrad achieves
\begin{multline*}
\frac{1}{T} \sum_{t=1}^T \mathbb{E}\left[ \| \nabla f(x_t) \|^2 \right] \leq \frac{1}{\sqrt{T}}\bigg[ \frac{G}{\alpha}\left(f(x_1) - f(x_\star)\right) \\
+ \frac{G^3}{(1-\beta_1)}\| \hat{v}_0^{-1/2} \|_1+ \frac{G^3d}{4L\alpha(1-\beta_1)} \\ 
+\frac{2GLd\alpha(1-\beta_1)(1+\log T)}{(1-\beta_2)(1-\gamma)}\bigg].
\end{multline*}
\end{theorem}
Compared with~\citep[Corollary 3.1]{chen2019convergence}, the initial value of $v_0 = \epsilon$ only affects one of the terms in our bound, whereas $\frac{1}{\epsilon}$ appears in all the terms of~\citep[Corollary 3.1]{chen2019convergence}.
The reason is that~\cite{chen2019convergence} uses  $v_0 \geq
\epsilon$ in many places of the proof, even when it was unnecessary.

Compared with~\citep[Corollary 3.9]{zhou2018convergence}, our result allows for bigger values of $\beta_1$, since we require $\beta_1^2 \leq \beta_1 < 1$ whereas~\citep[Corollary 3.9]{zhou2018convergence} requires $\beta_1 \leq \beta_2 < 1$.
Moreover,~\citep[Corollary 3.9]{zhou2018convergence} has a constant step size $\alpha = \frac{1}{\sqrt{dT}}$ that requires setting a horizon and becomes very small with large $d$.

Lastly, we have a $\log T$ dependence, whereas~\citep[Corollary
3.9]{zhou2018convergence} does not.  However, this is not for free and
it stems from the choice of a constant step size $\a_t=\frac{1}{\sqrt{dT}}$ therein.  In fact, it
is well known that for online gradient descent analysis, $\log T$ can be
shaved when  $\a_t\approx\frac{1}{\sqrt T}$.
However, in practice using a variable step size is more favorable,
since it does not require setting $T$ in advance.
Therefore, we choose to work with variable step size and
have the $\log T$ term in the bound.
\subsection{Flexible $\beta_1$ schedules}
We have focused on the case of constant $\beta_1$ throughout our paper,
as it is the most popular choice in practice.
However, it is possible that in some applications, practitioners might see benefit of using other schedules.
For instance, one can decrease $\beta_1$ until some threshold and keep it constant afterwards.
This is not covered by the previous regret analyses as $\beta_1$ needed to decrease to $0$.
With our framework however, one can use not only constant $\beta_1$, but any schedule as long as it is nonincreasing, and optimal regret bounds will follow.

Due to space constraints, we do not repeat all the proofs with this modification, but illustrate the main change that happens with variable $\beta_1$ and show that our proofs will go through.
In this section we switch to notation of $\beta_{1t}$ to illustrate time-varying case.

We start from the result of Lemma~\ref{lem: decop}, after summing over $t=1, \ldots, T$
\begin{multline}
\sum_{t=1}^T \langle g_t, x_t - x \rangle = \sum_{t=1}^T \langle m_{t-1}, x_{t-1} - x \rangle \\
+\sum_{t=1}^T \frac{1}{1-\beta_{1t}} \left( \langle m_t, x_t - x\rangle - \langle m_{t-1}, x_{t-1} - x\rangle \right)\\
-\sum_{t=1}^T\frac{\beta_{1t}}{1-\beta_{1t}} \langle m_{t-1}, x_t - x_{t-1} \rangle. \label{eq: anybeta}
\end{multline}

For bounding the terms on the first and third lines of~\eqref{eq: anybeta}, the only place that will change with varying $\beta_{1t}$ in the proof, is that $\alpha_t \| m_t \|^2_{\hat{v}_t^{-1/2}}$ will have a slightly different estimation, since now $m_t = \sum_{j=1}^t \prod_{k=1}^{t-j} \beta_{1(t-k+1)}(1-\beta_{1j})g_{j}^2$.
One can use that $\beta_{1t}\leq \beta_1$ to obtain the same bounds, but with $\frac{1}{(1-\beta_1)^2}$ factor multiplying the bounds now.
As explained before, this is one thing we lose with varying $\beta_{1t}$ in theory.

Next, we estimate the terms in the second line of~\eqref{eq: anybeta}
\begin{multline}
\frac{1}{1-\beta_{1t}}\left( \langle m_t, x_t - x \rangle - \langle m_{t-1}, x_{t-1} - x \rangle \right) = \\
\frac{1}{1-\beta_{1t}}\langle m_t, x_t - x \rangle - \frac{1}{1-\beta_{1(t-1)}}\langle m_{t-1}, x_{t-1} -x \rangle \\
+\left( \frac{\beta_{1(t-1)}-\beta_{1t}}{(1-\beta_{1t})(1-\beta_{1(t-1)})} \right) \langle m_{t-1}, x_{t-1} -x \rangle.\notag
\end{multline}

Now, for the last line we use that $\beta_{1t}$ is non-increasing, $\beta_{1t} \leq \beta_1$, $\|m_t\|_1 \leq dG$ and $\|x_t - x \|_{\infty} \leq D$, to get
\begin{multline}
\left( \frac{\beta_{1(t-1)}-\beta_{1t}}{(1-\beta_{1t})(1-\beta_{1(t-1)})} \right) \langle m_{t-1}, x_{t-1} -x \rangle\\
 \leq \frac{dDG}{(1-\beta_1)^2} \left( \beta_{1(t-1)}-\beta_{1t} \right).
\end{multline}

Thus upon summation over $t=1$ to $T$, as $m_0 = 0$,
\begin{multline}
\sum_{t=1}^T\frac{1}{1-\beta_{1t}} \left( \langle m_t, x_t - x\rangle - \langle m_{t-1}, x_{t-1} - x\rangle \right) \leq \\
\frac{1}{1-\beta_T} \langle m_T, x_T - x \rangle + \frac{dDG}{(1-\beta_1)^2} (\beta_{10} - \beta_{1T}),
\end{multline}
where we let $\beta_{10}=\beta_{11} < 1$.
Indeed, the contribution of this term will only be constant as $(1-\beta_{1t}) \leq 1, \forall t$, $\| m_t \|_{\infty}\leq G$, $\| x_t - x \|_{\infty}\leq D$.

Note that the estimation of the terms on the first and third
lines of~\eqref{eq: anybeta} are the same, as in the constant
$\beta_1$ case (up to constants).  Also, the contribution of the terms in
the second line of~\eqref{eq: anybeta} with varying $\beta_{1t}$ is a
constant. Thus, one can repeat our proofs, with any nonincreasing
$\beta_{1t}$ schedule and obtain the same optimal regret bounds, but with
slightly worse constants (compared to constant $\beta_1$ case).
 
 \section*{Acknowledgements}
This project has received funding from the European Research Council (ERC) under the European Union's Horizon $2020$ research and innovation programme (grant agreement no $725594$ - time-data), the Swiss National Science Foundation (SNSF) under grant number $200021\_178865 / 1$, the Department of the Navy, Office of Naval Research (ONR)  under a grant number N62909-17-1-211.
PM acknowledges financial support from the French National Research Agency
(ANR) under grant ORACLESS (ANR-16-CE33-0004-01) and the COST Action CA16229 ``European Network for Game Theory'' (GAMENET).

\nocite{*}
   
\bibliography{lit}
\bibliographystyle{icml2020}

\appendix 
\allowdisplaybreaks
\onecolumn
\section{Proofs}
\begin{proof}[Proof of \Cref{lem: decop}] 
    By  definition of $m_t$, $g_t = \frac{1}{1-\beta_1} m_t - \frac{\beta_1}{1-\beta_1} m_{t-1}$.
Thus, we have
\begin{align*}
\langle g_t, x_t - x \rangle &= \frac{1}{1-\beta_1} \langle m_t, x_t - x \rangle - \frac{\beta_1}{1-\beta_1} \langle m_{t-1}, x_t - x \rangle \\
&=\frac{1}{1-\beta_1} \langle m_t, x_t - x \rangle - \frac{\beta_1}{1-\beta_1} \langle m_{t-1}, x_{t-1} - x \rangle - \frac{\beta_1}{1-\beta_1} \langle m_{t-1}, x_t - x_{t-1} \rangle \\
&=\frac{1}{1-\beta_1}\big( \langle m_t, x_t - x \rangle - \langle m_{t-1}, x_{t-1} - x \rangle \big) + \langle m_{t-1}, x_{t-1} - x \rangle - \frac{\beta_1}{1-\beta_1} \langle m_{t-1}, x_t - x_{t-1} \rangle.
\end{align*}
\end{proof}
\subsection{Proofs for \amsgrad}
First, we need a useful inequality.
\begin{lemma}[Generalized H\"older inequality,
    \citealp{beckenbachinequalities}, Chap.~1.18]
    \label{lemma:gen_H}
    For $x,y,z\in \R^n_+$ and positive $p,q,r$ such that
    $\frac{1}{p}+\frac{1}{q}+\frac{1}{r}=1$, we have
    \[ \sum_{j=1}^nx_jy_jz_j\leq \n{x}_p\n{y}_q\n{z}_r. \]
\end{lemma}
The above lemma is used to obtain a slightly tighter bound for $ \|
m_t \|^2_{\hv_t^{-1/2}}$, compared to the standard analysis.

\begin{lemma}[Bound for $ \| m_t \|^2_{\hv_t^{-1/2}}$]
\label{lem: grad_norm}
Under Assumption~\ref{as:1}, $\beta_1 < 1$, $\beta_2 < 1$, $\gamma = \frac{\beta_1^2}{\beta_2} < 1$, $\epsilon > 0$, and the definitions of $\alpha_t$, $m_t$, $v_t$, $\hat v_t$ in \amsgrad, it holds that
    \begin{equation}\label{eq:bound_m_t}
     \| m_t \|^2_{\hv_t^{-1/2}}\leq
     \frac{(1-\beta_1)^2}{\sqrt{(1-\beta_2)(1-\gamma)}}\sum_{i=1}^d\sum_{j=1}^t\b_1^{t-j}|g_{j,i}|.
 \end{equation}
\end{lemma}
\begin{proof}
From the definition of $m_t$ and $v_t$, it follows that
\begin{align}
 m_t =(1-\beta_{1}) \sum_{j=1}^t \beta_1^{t-j} g_j, \qquad \qquad v_t =(1-\beta_{2}) \sum_{j=1}^t \beta_2^{t-j} g_j^2 \label{eq: mt_vt}.
\end{align}
Then we have
\begin{align*}
 \| m_t \|^2_{\hv_t^{-1/2}} &\leq \| m_t \|^2_{v_t^{-1/2}} = \sum_{i=1}^d  \frac{m_{t,i}^2}{{v}_{t, i}^{1/2}} = \sum_{i=1}^d  \frac{\left(\sum_{j=1}^t(1-\beta_{1}) \beta_{1}^{^{t-j}} g_{j, i} \right)^2}{\sqrt{\sum_{j=1}^t (1-\beta_{2})\beta_{2}^{t-j}g_{j, i}^2}} \notag\\
&= \frac{(1-\beta_1)^2}{\sqrt{1-\beta_2}} \sum_{i=1}^d \frac{\left(\sum_{j=1}^t \beta_{1}^{t-j} g_{j, i} \right)^2}{\sqrt{\sum_{j=1}^t \beta_{2}^{t-j}g_{j, i}^2}} \notag\\
&\leq \frac{(1-\beta_1)^2}{\sqrt{1-\beta_2}} \sum_{i=1}^d \frac{\bigg[\left(\sum_{j=1}^t (\beta_2^{\frac{t-j}{4}}|g_{j,i}|^{\frac{1}{2}})^4 \right)^{\frac{1}{4}} \left(\sum_{j=1}^t (\beta_1^{1/2}\b_2^{-1/4})^{4(t-j)}\right)^{\frac{1}{4}} \left(\sum_{j=1}^t(\b_1^{t-j}|g_{j,i}|)^{\frac{1}{2}\cdot2}\right)^{\frac{1}{2}}\bigg]^2}{\sqrt{\sum_{j=1}^t \beta_{2}^{t-j}g_{j, i}^2}} \notag\\&=\frac{(1-\beta_1)^2}{\sqrt{1-\beta_2}} \sum_{i=1}^d\left(\sum_{j=1}^t \gamma^{t-j}\right)^{\frac{1}{2}} \sum_{j=1}^t\b_1^{t-j}|g_{j,i}|\notag\\&\leq\frac{(1-\beta_1)^2}{\sqrt{(1-\beta_2)(1-\gamma)}}\sum_{i=1}^d\sum_{j=1}^t\b_1^{t-j}|g_{j,i}|,
\end{align*}
where the first inequality follows from the fact that $\hv^{1/2}_{t,i}\geq
v^{1/2}_{t,i}$, the second one follows from the generalized
H\"older inequality (\Cref{lemma:gen_H}) for
\[x_j=\beta_2^{\frac{t-j}{4}}|g_{j,i}|^{\frac{1}{2}},\quad
    y_j=(\beta_1\b_2^{-1/2})^{\frac{t-j}{2}},\quad
    z_j=(\beta_1^{t-j}|g_{j,i}|)^{\frac{1}{2}} \quad \text{and}\quad
    p=q=4, \quad r=2,\] and the third one follows from the sum
of geometric series and the assumption
$\gamma = \frac{\b_1^2}{\b_2}<1$.
 
We now comment on the possibility of observing many zero gradients in the beginning, causing $v_t = 0$ until some $t$, which would cause the appearance of the indeterminate form $\frac{0}{0}$ in the upper bound derived above~--- specifically in the term $\frac{m_{t, i}^2}{v_{t,i}^{1/2}}$.
For this, we will use the convention $\frac{0}{0} = 0$,
in which case the above derivations are always well-defined. For this, we argue as
follows: recall first that $v_{t,i} = 0$ iff $g_{j,i}=0$ for all
$j=1,\dots,t$. This being the case, we also get $m_{t,i} = 0$, and hence, $\frac{m_{t,i}^2}{{v}_{t, i}^{1/2}} = 0$. In fact,
this was done only for convenience, since $\hv_{t,i}\geq \epsilon$ and we
can always exclude zero terms from $\| m_t \|^2_{\hv_t^{-1/2}}$, before using the first line in the above chain of inequalities.
\end{proof}

\begin{lemma}[Bound for $\sum_{t=1}^T \alpha_t \| m_t
    \|^2_{\hv_t^{-1/2}}$]
    \label{lem: sum_grad_norm}
Under Assumption~\ref{as:1}, $\beta_1 < 1$, $\beta_2 < 1$, $\gamma = \frac{\beta_1^2}{\beta_2} < 1$, $\epsilon > 0$, and the definitions of $\alpha_t$, $m_t$, $v_t$, $\hat v_t$ in \amsgrad, we have
    \begin{equation}
        \label{eq: sum of m_t}
        \sum_{t=1}^T \alpha_t \| m_t
        \|^2_{\hv_t^{-1/2}}\leq\frac{(1-\beta_1)\alpha\sqrt{1+\log
                T}}{\sqrt{(1-\beta_2)(1-\gamma)}} \sum_{i=1}^d \sqrt{\sum_{t=1}^T
            g_{t,i}^2}.
    \end{equation}
\end{lemma}
\begin{proof}
    We have
    \begin{align*}
      \sum_{t=1}^T \alpha_t \| m_t \|^2_{\hv_t^{-1/2}} &\leq
                                                       \frac{(1-\beta_1)^2}{\sqrt{(1-\beta_2)(1-\gamma)}}
                                                       \sum_{i=1}^d
                                                       \sum_{t=1}^T
                                                       \alpha_t\sum_{j=1}^t
                                                       \b_1^{t-j} \vert
                                                       g_{j,i}\vert
      &&\text{(\Cref{eq:bound_m_t})}\\ &=\frac{(1-\beta_1)^2}{\sqrt{(1-\beta_2)(1-\gamma)}}
                                         \sum_{i=1}^d \sum_{j=1}^T \sum_{t=j}^T \alpha_t \b_1^{t-j} \vert
                                         g_{j,i}\vert &&\text{(Changing order of summation)}\\
                                                     &\leq \frac{(1-\beta_1)}{\sqrt{(1-\beta_2)(1-\gamma)}}  \sum_{i=1}^d
                                                       \sum_{j=1}^T
                                                       \alpha_j  \vert
                                                       g_{j,i}\vert &&\text{\big(Using}\sum_{t=j}^T\alpha_t\b_1^{t-j}\leq\frac{\a_j}{1-\b_1}\big)\\
                                                     &\leq
                                                       \frac{1-\beta_1}{\sqrt{(1-\beta_2)(1-\gamma)}}
                                                       \sum_{i=1}^d
                                                       \sqrt{\sum_{j=1}^T
                                                       \alpha_j^2}\sqrt{\sum_{j=1}^T
                                                       g_{j,i}^2 } && \text{(Cauchy-Schwarz)}\\
                                                     &\leq\frac{(1-\beta_1)\alpha\sqrt{1+\log
                                                       T}}{\sqrt{(1-\beta_2)(1-\gamma)}}
                                                       \sum_{i=1}^d
                                                       \sqrt{\sum_{t=1}^T
                                                       g_{t,i}^2} &&
                                                                     \text{\big(Using
                                                                     }
                                                                     \sum_{j=1}^T\frac{1}{j}\leq
                                                                     1+\log
                                                                     T\big).
    \qedhere
    \end{align*}
\end{proof}
We now restate Theorem~\ref{th:amsgrad} for easy navigation and proceed to its proof.
\begin{reptheorem}{th:amsgrad}
Under Assumption~\ref{as:1}, $\beta_1 < 1$, $\beta_2 < 1$, $\gamma = \frac{\beta_1^2}{\beta_2} < 1$, and $\epsilon > 0$, \hyperref[alg:amsgrad]{\amsgrad} achieves the regret
\begin{equation}
  R(T)  \leq \frac{D^2\sqrt{T}}{2\a(1-\b_1)}\sum_{i=1}^d \hv^{1/2}_{T,i}\\ +\frac{\alpha\sqrt{1+\log T}}{\sqrt{(1-\beta_2)(1-\gamma)}} \sum_{i=1}^d
  \sqrt{\sum_{t=1}^T g_{t,i}^2}.\notag
\end{equation}
\end{reptheorem}
\begin{proof}
    Let $x\in \argmin_{y\in\cX} \sum_{t=1}^Tf_t(y)$. Then by
    convexity, we immediately have
    \[R(T)\leq \sum_{t=1}^T\lr{g_t,x_t-x}.\]
    Hence, our goal is to bound the latter expression.  If we sum the
    inequality from \Cref{lem: decop} over $t=1,\dots, T$ and use the fact that
    $m_0=0$, we obtain
    \begin{align}
\sum_{t=1}^T \langle g_t, x_t - x \rangle &= \frac{1}{1-\beta_1} \big( \langle m_T, x_T - x \rangle - \langle m_0, x_0 - x \rangle \big) + \langle m_0, x_0 - x \rangle + \sum_{t=1}^{T-1} \langle m_t, x_t - x \rangle \notag\\
&\qquad +\frac{\beta_1}{1-\beta_1} \sum_{t=1}^T \langle m_{t-1}, x_{t-1} - x_t \rangle\notag\\
&= \frac{\beta_1}{1-\beta_1} \langle m_T, x_T - x \rangle  + \sum_{t=1}^{T} \langle m_t, x_t - x \rangle+\frac{\beta_1}{1-\beta_1} \sum_{t=1}^T \langle m_{t-1}, x_{t-1} - x_t \rangle.
      \label{eq: new_ams_allterms}
\end{align}
We will separately bound each term in the right-hand side
of~\eqref{eq: new_ams_allterms} and then combine these bounds
together.

$\bullet$ \emph{Bound for $\sum_{t=1}^{T} \langle m_t, x_t - x \rangle$}.

As $x\in\mathcal{X}$, by the nonexpansiveness property~\eqref{eq: nonexp}, we get
\begin{align}\label{eq: new_ams_thirdterm_1}
\| x_{t+1} - x\|^2_{\hat{v}_t^{1/2}} &= \| {P}_{\mathcal{X}}^{\hat{v}_t^{1/2}} \left(x_{t} - \alpha_t \hat{v}_t^{-1/2} m_t \right) - x\|^2_{\hat{v}_t^{1/2}} \notag \\
&\leq \| x_{t} - \alpha_t \hat{v}_t^{-1/2} m_t - x\|^2_{\hat{v}_t^{1/2}}\notag \\
&= \| x_t - x \|^2_{\hat{v}_t^{1/2}} - 2\alpha_t \langle m_t, x_t - x\rangle + \| \alpha_t \hat{v}_t^{-1/2} m_t \|^2_{\hat{v}_t^{1/2}}\notag\\&= \| x_t - x \|^2_{\hat{v}_t^{1/2}} - 2\alpha_t \langle m_t, x_t - x \rangle + \a_t^2\| m_t \|^2_{\hat{v}_t^{-1/2}}.
\end{align} 

We rearrange and divide both sides of \eqref{eq: new_ams_thirdterm_1} by $2\alpha_t$ to get
\begin{align}
\langle m_t, x_t - x\rangle &\leq \frac{1}{2\a_t} \| x_t - x \|^2_{\hat{v}_t^{1/2}} - \frac{1}{2\a_t}  \| x_{t+1}-x\|^2_{\hat{v}_t^{1/2}} + \frac{\alpha_t}{2} \| m_t \|^2_{\hv_t^{-1/2}} \notag\\
 &= \frac{1}{2\a_{t-1}} \| x_t - x \|^2_{\hat{v}_{t-1}^{1/2}} - \frac{1}{2\a_t}\| x_{t+1} - x \|^2_{\hat{v}_t^{1/2}} + \frac{1}{2}\sum_{i=1}^d \left( \frac{\hat{v}_{t, i}^{1/2}}{\alpha_t} - \frac{\hat{v}_{t-1, i}^{1/2}}{\alpha_{t-1}} \right) ( x_{t, i} - x_i )^2 + \frac{\alpha_t}{2} \| m_t \|^2_{\hv_t^{-1/2}} \notag\\
&\leq \frac{1}{2\a_{t-1}} \| x_t - x \|^2_{\hat{v}_{t-1}^{1/2}} - \frac{1}{2\a_t}\| x_{t+1} - x \|^2_{\hat{v}_t^{1/2}} + \frac{D^2}{2} \sum_{i=1}^d \left( \frac{\hat{v}_{t, i}^{1/2}}{\alpha_t} - \frac{\hat{v}_{t-1, i}^{1/2}}{\alpha_{t-1}} \right) + \frac{\alpha_t}{2} \| m_t \|^2_{{\hv}_t^{-1/2}}, \label{eq: new_ams_reg_bd}
\end{align} 
where the last inequality is due to the fact that
$\hat{v}_{t, i} \geq \hat{v}_{t-1, i}$,
$\frac{1}{\a_t}\geq \frac{1}{\a_{t-1}}$, and the definition of
$D$.%
\footnote{Note that for  $t=1$  we suppose that $\frac{1}{\a_0}=0$; this
makes the above derivation still valid, as $\alpha_0$ is not used in the algorithm, and this is only for convenience.}

Summing \eqref{eq: new_ams_reg_bd} over $t=1,\dots T$ and using that
$\frac{1}{2\a_{0}} \| x_1 - x \|^2_{\hat{v}_{0}^{1/2}} = 0$ yields
\begin{align}
\sum_{t=1}^T \langle m_t, x_t - x\rangle &\leq \frac{D^2}{2\alpha_T} \sum_{i=1}^d \hat{v}_{T, i}^{1/2} +\frac{1}{2} \sum_{t=1}^T \alpha_t \| m_t \|^2_{\hv_t^{-1/2}}.\label{eq: new_ams_linear2}
\end{align}

$\bullet$ \emph{Bound for $\sum_{t=1}^T \langle m_{t-1}, x_{t-1} - x_t \rangle$}.

Now let us bound the last term in~\eqref{eq: new_ams_allterms}.
\begin{align}
  \sum_{t=1}^T \langle m_{t-1}, x_{t-1} - x_t \rangle &=
\sum_{t=2}^T \langle m_{t-1}, x_{t-1} - x_t \rangle=\sum_{t=1}^{T-1} \langle m_{t}, x_{t} - x_{t+1} \rangle&&\text{(Using
$m_0=0$)}\notag\\
&\leq \sum_{t=1}^{T-1} \| m_{t}\|_{\hat{v}_t^{-1/2}} \|
x_{t+1} - x_{t} \|_{\hat{v}_{t}^{1/2}} && \text{(H\"older
inequality)}\notag \\ 
&= \sum_{t=1}^{T-1} \| m_{t}\|_{\hat{v}_t^{-1/2}}
\Big\|P_{\mathcal{X}}^{\hat{v}_{t}^{1/2}}\left( x_{t} - \alpha_{t}
\hat{v}_{t}^{-1/2}m_{t}\right)-P_{\mathcal{X}}^{\hat{v}_{t}^{1/2}} (x_{t}) \Big\|_{\hat{v}_{t}^{1/2}} &&
\text{(Using $x_{t}\in \cX $)}\notag \\ 
&\leq
\sum_{t=1}^{T-1}\alpha_{t} \| m_{t}\|_{\hat{v}_{t}^{-1/2}} \|
\hat{v}_{t}^{-1/2}m_{t} \|_{\hat{v}_{t}^{1/2}}
&&\text{(Nonexpansiveness of $P_{\cX}^{\hat{v}_{t}^{1/2}}$)}\notag\\
&= \sum_{t=1}^{T-1}\alpha_{t} \|m_{t}\|_{\hat{v}_{t}^{-1/2}}^2&&\text{(Property $\n{u^{-1}x}_u=\n{x}_{u^{-1}}$)}\label{eq: sum_of_subseq}.
\end{align}
At this point, we could use \cref{eq: sum of m_t} to obtain a final bound for
$\sum_{t=1}^T\lr{m_{t-1},x_{t-1}-x_t}$. However, we postpone it to combine it with the term $\lr{m_T,x_T-x}$ in \eqref{eq:
    new_ams_allterms} to have a shorter expression.

$\bullet$ \emph{Bound for $\lr{m_T,x_T-x}$}.

This term is the easiest for estimation:
\begin{align}\label{eq: sum_of_single}
    \lr{m_T,x_T-x}&\leq \n{m_T}_{\hv_T^{-1/2}}\n{x_T-x}_{\hv_T^{1/2}}&&\text{(H\"older's inequality)}\notag\\ 
    &\leq
\a_T\n{m_T}_{\hv_T^{-1/2}}^2 +
\frac{1}{4\a_T} \n{x_T-x}^2_{\hv_T^{1/2}} &&\text{(Young's inequality)}\notag\\ 
&\leq \a_T\n{m_T}_{\hv_T^{-1/2}}^2+
\frac{D^2}{4\a_T} \sum_{i=1}^d \hv^{1/2}_{T,i}&&\text{(Definition of $D$)}
\end{align}
We now have all the ingredients required to bound the right-hand side of
\eqref{eq: new_ams_allterms}. To that end, after all substitutions and some
straightforward algebra, we obtain
\begin{align} \text{RHS of \eqref{eq: new_ams_allterms}} &=
    \frac{\beta_1}{1-\beta_1}\left( \langle m_T, x_T - x \rangle +
        \sum_{t=1}^T \langle m_{t-1}, x_{t-1} - x_t \rangle\right) +
    \sum_{t=1}^{T} \langle m_t, x_t - x \rangle\notag\\ 
    &\leq
    \frac{\beta_1}{1-\beta_1}\left( \frac{D^2}{4\a_T}\sum_{i=1}^d
        \hv^{1/2}_{T,i} + \sum_{t=1}^{T}\alpha_{t} \|
        m_{t}\|_{\hat v_{t}^{-1/2}}^2\right) + \frac{D^2}{2\alpha_T}
    \sum_{i=1}^d \hat{v}_{T, i}^{1/2} +\frac{1}{2} \sum_{t=1}^T
    \alpha_t \| m_t \|^2_{\hat v_t^{-1/2}} \notag\\
    &=
    \frac{(2-\b_1)D^2}{4\a_T(1-\b_1)}\sum_{i=1}^d \hv^{1/2}_{T,i} +
    \frac{1+\b_1}{2(1-\b_1)} \sum_{t=1}^T \alpha_t \| m_t
    \|^2_{\hat v_t^{-1/2}} \notag\\
    &\leq
    \frac{D^2\sqrt{T}}{2\a(1-\b_1)}\sum_{i=1}^d \hv^{1/2}_{T,i} +
    \frac{1}{1-\b_1} \sum_{t=1}^T \alpha_t \| m_t
    \|^2_{\hat v_t^{-1/2}} \notag\\
    &\leq \frac{D^2\sqrt{T}}{2\a(1-\b_1)}\sum_{i=1}^d \hv^{1/2}_{T,i} +\frac{\alpha\sqrt{1+\log T}}{\sqrt{(1-\beta_2)(1-\gamma)}} \sum_{i=1}^d
    \sqrt{\sum_{t=1}^T g_{t,i}^2},\label{eq: ams_final_est}
\end{align}
where the second inequality follows from the assumption $\frac{2-\b_1}{4}\leq \frac
12$, $\frac{1+\b_1}{2}\leq 1$, and $\alpha_T = \frac{\alpha}{\sqrt{T}}$,
and the last follows by \Cref{lem: sum_grad_norm}.
\end{proof}

\subsection{Proofs for \adamnc}
We first give analogous results to \Cref{lem: grad_norm,lem: sum_grad_norm}, which are mostly standard and
simplified thanks to a constant $\beta_1$.
\begin{lemma}[Bound for $\| m_t \|^2_{v_t^{-1/2}}$]
\label{lem: nc_grad_norm}
Under Assumption~\ref{as:1}, $\beta_1 < 1$, $\epsilon > 0$, and the definitions of $\alpha_t$, $m_t$, $v_t$ in~\adamnc, it holds that
\begin{equation}
\| m_t \|^2_{v_t^{-1/2}} \leq \sqrt{t}(1-\beta_1) \sum_{i=1}^d\sum_{j=1}^t \frac{\beta_1^{t-j} g_{j, i}^2}{\sqrt{\sum_{k=1}^j g_{k, i}^2}}. \notag
\end{equation}
\end{lemma}
\begin{proof}
    Using the expression \eqref{eq: mt_vt} for $m_t$ and
    $v_{t,i}=\frac{1}{t}\left(\sum_{j=1}^t  g_{j,i}^2 + \epsilon\right)$, we obtain:%
\footnote{In the sequel, the same comments about the indeterminate form $\frac{0}{0}$ apply here as in Lemma~\ref{lem: grad_norm}.}
\begin{align}
\| m_t \|^2_{v_t^{-1/2}} &= \sum_{i=1}^d \frac{m_{t, i}^2}{v_{t, i}^{1/2}} = \sum_{i=1}^d \frac{\left( \sum_{j=1}^t (1-\beta_1)\beta_1^{t-j} g_{j, i} \right)^2}{\sqrt{\frac{\epsilon}{t} + \frac{1}{t} \sum_{k=1}^t g_{k, i}^2}}\notag\\
&\leq \sqrt{t}(1-\beta_1)^2 \sum_{i=1}^d \frac{\left( \sum_{j=1}^t \beta_1^{t-j} g_{j, i} \right)^2}{\sqrt{\sum_{k=1}^t g_{k, i}^2}} \notag \\
&\leq \sqrt{t}(1-\beta_1)^2 \sum_{i=1}^d \frac{\left( \sum_{j=1}^t \beta_1^{t-j} g_{j, i}^2\right)\left(\sum_{j=1}^t\beta_1^{t-j}\right)}{\sqrt{\sum_{k=1}^t g_{k, i}^2}} \notag \\
&\leq \sqrt{t}(1-\beta_1) \sum_{i=1}^d \frac{ \sum_{j=1}^t \beta_1^{t-j} g_{j, i}^2}{\sqrt{\sum_{k=1}^t g_{k, i}^2}} \leq \sqrt{t}(1-\beta_1) \sum_{i=1}^d\sum_{j=1}^t \frac{\beta_1^{t-j} g_{j, i}^2}{\sqrt{\sum_{k=1}^j g_{k, i}^2}},
%\notag \\
\end{align}
where the first inequality is due to $\epsilon > 0$, second inequality
is by Cauchy-Schwarz, the third one by the sum of geometric series, and
the final one is by $j \leq t$.
\end{proof}

\begin{lemma}[Bound for $\sum_{t=1}^T\alpha_t \| m_t \|^2_{v_t^{-1/2}}$]
\label{eq: nc_sum_grad_norm}
Under Assumption~\ref{as:1}, $\beta_1 < 1$, $\epsilon > 0$, and the definitions of $\alpha_t$, $m_t$, $v_t$ in~\adamnc, it holds that
\begin{equation}
\sum_{t=1}^T \alpha_t \| m_t \|^2_{v_t^{-1/2}} \leq 2\alpha\sum_{i=1}^d \sqrt{\sum_{t=1}^T g_{t, i}^2}.
\end{equation}
\end{lemma}
\begin{proof}
We have, by using Lemma~\ref{lem: nc_grad_norm}
\begin{align}
\sum_{t=1}^T \alpha_t \| m_t \|^2_{v_t^{-1/2}} &= \sum_{t=1}^T \alpha_t \sqrt{t}(1-\beta_1) \sum_{i=1}^d\sum_{j=1}^t \frac{\beta_1^{t-j} g_{j, i}^2}{\sqrt{\sum_{k=1}^j g_{k, i}^2}} \notag \\
&=\alpha(1-\beta_1) \sum_{i=1}^d \sum_{t=1}^T \sum_{j=1}^t \frac{\beta_1^{t-j} g_{j, i}^2}{\sqrt{\sum_{k=1}^j g_{k, i}^2}} \notag \\
&= \alpha(1-\beta_1) \sum_{i=1}^d \sum_{j=1}^T \sum_{t=j}^T \frac{\beta_1^{t-j} g_{j, i}^2}{\sqrt{\sum_{k=1}^j g_{k, i}^2}} \notag \\
&\leq \alpha \sum_{i=1}^d \sum_{j=1}^T \frac{g_{j, i}^2}{\sqrt{\sum_{k=1}^j g_{k, i}^2}} \notag \\
&\leq 2\alpha \sum_{i=1}^d \sqrt{\sum_{j=1}^T g_{j, i}^2}, \notag
\end{align}
where the second equality is due to $\alpha_t = \frac{\alpha}{\sqrt{t}}$, third equality is by changing the order of summation, first inequality by summation of the geometric series.
For the last inequality, we use a standard inequality for numerical sequences, encountered for example in~\citet[Lemma~3.5]{auer2002adaptive}
\begin{equation*}
\sum_{j=1}^T \frac{a_j}{\sqrt{\sum_{k=1}^j a_k}} \leq 2\sqrt{\sum_{j=1}^T a_j}
	\quad
	\text{for all $a_{1},\dotsc,a_{T}\geq0$.}
	\qedhere
\end{equation*}
\end{proof}

We now restate Theorem~\ref{th: adamnc} and present its proof.

\begin{reptheorem}{th: adamnc}
Under Assumption~\ref{as:1}, $\beta_1 < 1$, and $\epsilon > 0$, \hyperref[alg:adamnc]{\adamnc} enjoys the regret bound
\begin{align}
R(T) &\leq \frac{D^2\sqrt{T}}{2\alpha(1-\beta_1)} \sum_{i=1}^d v_{T, i}^{1/2}+ \frac{2\alpha}{1-\beta_1}\sum_{i=1}^d \sqrt{\sum_{t=1}^T g_{t, i}^2}.\notag
\end{align}
\end{reptheorem}

\begin{proof} We will follow the proof structure of Theorem~\ref{th:amsgrad}.
First, we start from~\eqref{eq: new_ams_allterms} which applies to \adamnc as the update of $m_t$ is the same as~\amsgrad
\begin{align}
R(T)\leq \sum_{t=1}^T \langle g_t, x_t - x \rangle &= \frac{\beta_1}{1-\beta_1} \langle m_T, x_T - x \rangle + \sum_{t=1}^T \langle m_t, x_t - x \rangle + \frac{\beta_1}{1-\beta_1} \sum_{t=1}^T \langle m_{t-1}, x_{t-1}-x_t \rangle.\label{eq: nc_main_ineq1}
\end{align}
Then we again bound each term in the right-hand side seperately.

$\bullet$ \emph{Bound for $\sum_{t=1}^T \langle m_t, x_t - x \rangle$.}

We proceed similarly to the derivations in~\eqref{eq: new_ams_thirdterm_1}
and~\eqref{eq: new_ams_reg_bd}, the main change being that we now have $v_t$ instead of $\hat{v}_t$.
We have:
\begin{align}
\langle m_t, x_t - x \rangle &\leq \frac{1}{2\a_{t-1}} \| x_t - x \|^2_{{v}_{t-1}^{1/2}} - \frac{1}{2\a_t}\| x_{t+1} - x \|^2_{{v}_t^{1/2}} + \frac{1}{2}\sum_{i=1}^d \left( \frac{{v}_{t, i}^{1/2}}{\alpha_t} - \frac{{v}_{t-1, i}^{1/2}}{\alpha_{t-1}} \right) ( x_{t, i} - x_i )^2 + \frac{\alpha_t}{2} \| m_t \|^2_{v_t^{-1/2}} \notag\\
&\leq \frac{1}{2\a_{t-1}} \| x_t - x \|^2_{{v}_{t-1}^{1/2}} - \frac{1}{2\a_t}\| x_{t+1} - x \|^2_{{v}_t^{1/2}} + \frac{D^2}{2} \sum_{i=1}^d \left( \frac{{v}_{t, i}^{1/2}}{\alpha_t} - \frac{{v}_{t-1, i}^{1/2}}{\alpha_{t-1}} \right) + \frac{\alpha_t}{2} \| m_t \|^2_{{v}_t^{-1/2}},
\end{align}
where the last inequality is due to $\frac{v_{t, i}^{1/2}}{\alpha_t} \geq \frac{v_{t-1, i}^{1/2}}{\alpha_{t-1}}$, since by definition $v_{t, i} = \frac{1}{t} \left(\sum_{j=1}^t g_{j, i}^2+\epsilon\right)$ and $\alpha_t = \frac{\alpha}{\sqrt{t}}$.

We now proceed to telescope this inequality, assuming as before that $\frac{1}{\alpha_0} = 0$.
Doing so, we obtain:
\begin{equation}
\sum_{t=1}^T \langle m_t, x_t - x \rangle \leq \frac{D^2}{2} \sum_{i=1}^d \frac{v_{T, i}^{1/2}}{\alpha_T} + \frac{1}{2} \sum_{t=1}^T \alpha_t \| m_t \|^2_{v_t^{-1/2}}.\label{eq: nc_first_term}
\end{equation}

$\bullet$ \emph{Bounds for $\langle m_T, x_T - x\rangle$ and $\sum_{t=1}^T\langle m_{t-1}, x_{t-1} - x_t \rangle$}

These bounds will be similar as in the proof of Theorem~\ref{th:amsgrad}. Again, the only change
in calculations in~\eqref{eq: sum_of_subseq} and~\eqref{eq:
    sum_of_single} is that now we have $v_t$ instead of $\hat{v}_t$
\begin{equation}
\sum_{t=1}^T \langle m_{t-1}, x_{t-1}- x_t \rangle \leq \sum_{t=1}^{T-1} \alpha_t \| m_t \|^2_{v_t^{-1/2}},\label{eq: nc_sec_term}
\end{equation}
and
\begin{equation}
\langle m_T, x_T - x \rangle \leq \alpha_T \| m_T \|^2_{v_T^{-1/2}} + \frac{D^2}{4\alpha_T} \sum_{i=1}^d v_{T, i}^{1/2}.\label{eq: nc_third_term}
\end{equation}

We now combine~\eqref{eq: nc_first_term},~\eqref{eq: nc_sec_term},
and~\eqref{eq: nc_third_term} in~\eqref{eq: nc_main_ineq1}, estimate
using the same steps in~\eqref{eq: ams_final_est}, and use the bound
for $\sum_{t=1}^T \alpha_t \| m_t \|^2_{v_t^{-1/2}}$ from
Lemma~\ref{eq: nc_sum_grad_norm} to conclude:

\begin{align*}
\sum_{t=1}^T \langle g_t, x_t - x \rangle &= \frac{\beta_1}{1-\beta_1} \langle m_T, x_T - x \rangle + \sum_{t=1}^T \langle m_t, x_t - x \rangle + \frac{\beta_1}{1-\beta_1} \sum_{t=1}^T \langle m_{t-1}, x_{t-1}-x_t \rangle\notag \\
&\leq \left(\frac{D^2}{2} + \frac{\beta_1 D^2}{4(1-\beta_1)} \right)\sum_{i=1}^d \frac{v_{T, i}^{1/2}}{\alpha_T} + \left( \frac{1}{2} + \frac{\beta_1}{1-\beta_1} \right)\sum_{t=1}^T \alpha_t \| m_t \|^2_{v_t^{-1/2}} \notag \\
&\leq \frac{D^2\sqrt{T}}{2\alpha(1-\beta_1)} \sum_{i=1}^d v_{T, i}^{1/2}+ \frac{2\alpha}{1-\beta_1}\sum_{i=1}^d \sqrt{\sum_{t=1}^T g_{t, i}^2}.
\qedhere
\end{align*}
\end{proof}

\subsection{Proofs for \sadam}
\begin{lemma}[Bound for $\| m_t \|^2_{\hat{v}_t^{-1}}$]
\label{lem: first_lem_sadam}
Under Assumption~\ref{as:1}, $\beta_1 < 1$, $\epsilon > 0$, and the definitions of $\alpha_t$, $m_t$, $v_t$, $\hat v_t$ in \sadam, it holds that
\begin{equation}\label{eq:sadam_mt}
\| m_t \|^2_{\hat{v}_t^{-1}} \leq t (1-\beta_1) \sum_{i=1}^d \sum_{j=1}^t \frac{\beta_1^{t-j}g_{j, i}^2}{\sum_{k=1}^j g_{k, i}^2 + \epsilon}.
\end{equation}
\end{lemma}
\begin{proof}
We have
\begin{align}
\| m_t \|^2_{\hat{v}_t^{-1}} &= \sum_{i=1}^d \frac{m_{t, i}^2}{\hat{v}_{t, i}} = \sum_{i=1}^d \frac{m_{t, i}^2}{{v}_{t, i} + \frac{\epsilon}{t}} =t(1-\beta_1)^2 \sum_{i=1}^d \frac{\left(\sum_{j=1}^t \beta_1^{t-j} g_{j, i}\right)^2}{\sum_{k=1}^t g_{k, i}^2+\epsilon} \notag \\
&\leq t(1-\beta_1) \sum_{i=1}^d \frac{\sum_{j=1}^t \beta_1^{t-j} g_{j, i}^2}{\sum_{k=1}^t g_{k, i}^2+\epsilon} \notag \\
&\leq t(1-\beta_1) \sum_{i=1}^d \sum_{j=1}^t \frac{ \beta_1^{t-j} g_{j, i}^2}{\sum_{k=1}^j g_{k, i}^2+\epsilon},
\end{align}
where we used the definitions $\hat{v}_{t, i} = \frac{1}{t}\sum_{k=1}^t g_{k, i}^2 + \frac{\epsilon}{t}$ and the expression for $m_t$ from~\eqref{eq: mt_vt} in the first line.
First inequality follows from Cauchy-Schwarz and sum of geometric series; and the last inequality is by  $j \leq t$.
\end{proof}

\begin{lemma}[Bound for $\sum_{t=1}^T \alpha_t \| m_t \|^2_{\hat{v}_t^{-1}}$]
\label{lem: sadam_sum_grad_norms}
Under Assumption~\ref{as:1}, $\beta_1 < 1$, $\epsilon > 0$, and the definitions of $\alpha_t$, $m_t$, $v_t$, $\hat v_t$ in \sadam, it holds that
\begin{equation}
\sum_{t=1}^T \alpha_t \| m_t \|^2_{\hat{v}_t^{-1}} \leq \alpha \sum_{i=1}^d \log\left( \frac{\sum_{t=1}^T g_{t, i}^2}{\epsilon}+1 \right).
\end{equation}
\end{lemma}
\begin{proof}
We have, by Lemma~\ref{lem: first_lem_sadam}
\begin{align}
\sum_{t=1}^T \alpha_t \| m_t \|^2_{\hat{v}_t^{-1}} &= \sum_{t=1}^T \alpha_t t (1-\beta) \sum_{i=1}^d \sum_{j=1}^t \frac{\beta_1^{t-j}g_{j, i}^2}{\sum_{k=1}^j g_{k, i}^2 + \epsilon} \notag \\
&= \alpha (1-\beta) \sum_{i=1}^d \sum_{t=1}^T \sum_{j=1}^t \frac{\beta_1^{t-j}g_{j, i}^2}{\sum_{k=1}^j g_{k, i}^2 + \epsilon} \notag \\
&= \alpha (1-\beta) \sum_{i=1}^d \sum_{j=1}^T \sum_{t=j}^T \frac{\beta_1^{t-j}g_{j, i}^2}{\sum_{k=1}^j g_{k, i}^2 + \epsilon} \notag \\
&\leq \alpha \sum_{i=1}^d \sum_{j=1}^T \frac{g_{j, i}^2}{\sum_{k=1}^j g_{k, i}^2 + \epsilon} \leq \alpha \sum_{i=1}^d \log\left( \frac{\sum_{t=1}^T g_{t, i}^2}{\epsilon} + 1 \right),
\end{align}
where the second equality is by the definition of $\alpha_t$ and the third equality is by changing the order of summation.
Moreover, first inequality is by the sum of geometric series and the last inequality is due to the fact that
\begin{align}
  \sum_{j=1}^T \frac{a_j}{\sum_{k=1}^j a_k + \epsilon} \leq \log
  \left(\frac{\sum_{j=1}^Ta_j}{\epsilon} + 1\right),
\end{align}
for nonnegative $a_1, \ldots, a_T$ and $\epsilon > 0$ \textendash\ see e.g.,
\citet[Lemma 12]{duchi2010adaptive_techreport} and \citet[Lemma 11]{hazan2007logarithmic}.
\end{proof}
We now restate Theorem~\ref{th: sadam} and present its proof.
\begin{reptheorem}{th: sadam}
Let Assumption~\ref{as:1} hold and $f_t$ be $\mu$-strongly
convex, $\forall t$. Then, if $\beta_1 < 1$, $\epsilon > 0$, and $\alpha \geq \frac{G^2}{\mu}$, \hyperref[alg:sadam]{\sadam} achieves
\begin{align}
R(T) \leq \frac{\beta_1 d G D}{1-\beta_1}  + \frac{\alpha}{1-\beta_1} \sum_{i=1}^d \log\left( \frac{\sum_{t=1}^T g_{t, i}^2}{\epsilon}+1 \right).\notag
\end{align}
\end{reptheorem}
\begin{proof}
    Let $x= \argmin_{y\in\cX} \sum_{t=1}^Tf_t(y)$.  In
    \Cref{th:amsgrad} we used convexity only once: going from
    $R(T)$ to $\sum_{t=1}^T\lr{g_t,x_t-x}$. Instead, strong convexity
    gives us
    $f_t(x) \geq f_t(x_t) + \langle g_t, x - x_t \rangle +
    \frac{\mu}{2} \| x_t - x \|^2$, which combined for all $t$ yields
    \begin{equation}\label{eq: sadam_regret}
        R(T)=\sum_{t=1}^T f_t(x_t)-f_t(x) \leq
        \sum_{t=1}^T\lr{g_t,x_t-x} -
        \frac{\mu}{2}\sum_{t=1}^T\n{x_t-x}^2.
    \end{equation}
    We want to estimate $\sum_{t=1}^T\lr{g_t,x_t-x}$. Similarly
to~\eqref{eq: new_ams_allterms}, we have
\begin{equation}
\sum_{t=1}^T \langle g_t, x_t - x \rangle \leq \frac{\beta_1}{1-\beta_1} \langle m_T, x_T - x \rangle + \sum_{t=1}^T \langle m_t, x_t - x\rangle + \frac{\beta_1}{1-\beta_1} \sum_{t=1}^T \langle m_{t-1}, x_{t-1} - x_t \rangle.\label{eq: sadam_main1}
\end{equation}

$\bullet$ \emph{Bound for $\sum_{t=1}^T \langle m_t, x_t - x \rangle$.}

We proceed similarly to~\eqref{eq: new_ams_thirdterm_1} and~\eqref{eq:
    new_ams_reg_bd}. The only change is that now we have $\hat{v}_t$ instead of $\hat{v}_t^{1/2}$
\begin{align}
 \langle m_t, x_t - x \rangle &\leq \frac{1}{2\a_{t-1}} \| x_t - x \|^2_{\hat{v}_{t-1}} - \frac{1}{2\a_t}\| x_{t+1} - x \|^2_{\hat{v}_t} + \frac{1}{2}\sum_{i=1}^d \left( \frac{\hat{v}_{t, i}}{\alpha_t} - \frac{\hat{v}_{t-1, i}}{\alpha_{t-1}} \right) ( x_{t, i} - x_i )^2 + \frac{\alpha_t}{2} \| m_t \|^2_{\hat{v}_t^{-1}}. \notag
\end{align}
We sum the above inequality and use the fact that $\frac{1}{\alpha_0} \| x_1 - x \|^2_{\hat{v}_0} = 0$ to obtain
\begin{align}
\sum_{t=1}^T \langle m_t, x_t - x \rangle &\leq \sum_{t=1}^T \sum_{i=1}^d \left( \frac{\hat{v}_{t, i}}{2\alpha_t} - \frac{\hat{v}_{t-1, i}}{2\alpha_{t-1}} \right) (x_{t, i}-x_i)^2 + \sum_{t=1}^T \frac{\alpha_t}{2} \| m_t \|^2_{\hat{v}_t^{-1}}.
\end{align}

$\bullet$ \emph{Bound for $\sum_{t=1}^T\langle m_{t-1}, x_{t-1} - x_t \rangle$}

This bound will be similar to the one we derived for Theorem~\ref{th:amsgrad}. The main change in the calculations of~\eqref{eq: sum_of_subseq} is that we will have $\hat{v}_t$ instead of $\hat{v}_t^{1/2}$ for using H\"older's inequality and nonexpansiveness
\begin{equation}
\sum_{t=1}^T \langle m_{t-1}, x_{t-1}- x_t \rangle \leq \sum_{t=1}^{T-1} \alpha_t \| m_t \|^2_{\hat{v}_t^{-1}}\leq \sum_{t=1}^{T} \alpha_t \| m_t \|^2_{\hat{v}_t^{-1}}.\label{eq: sadam_sec_term}
\end{equation}

We collect these estimations in~\eqref{eq: sadam_main1} and
\eqref{eq: sadam_regret} to derive
\begin{align}
R(T)=\sum_{t=1}^T f_t(x_t) - f_t(x) &\leq \frac{\beta_1}{1-\beta_1} \langle m_T, x_T - x \rangle + \frac{1+\beta_1}{2(1-\beta_1)} \sum_{t=1}^T \alpha_t \| m_t \|^2_{\hat{v}_t^{-1}} \notag \\
&+ \sum_{t=1}^T \sum_{i=1}^d \left( \frac{\hat{v}_{t, i}}{2\alpha_t} - \frac{\hat{v}_{t-1, i}}{2\alpha_{t-1}} \right) (x_{t, i}-x_i)^2 - \sum_{t=1}^T \sum_{i=1}^d \frac{\mu}{2} (x_{t, i} - x_i )^2.\label{eq: sadam_main2}
\end{align}
We collect the last two terms and use the assumption on the step size $\alpha \geq \frac{G^2}{\mu}$ and the
definition $\hat{v}_{t, i} = \frac{1}{t} \sum_{j=1}^t g_{j, i}^2 +
\frac{\epsilon}{t}$ to derive
\begin{align}
 \frac{\hat{v}_{t, i}}{2\alpha_t} - \frac{\hat{v}_{t-1, i}}{2\alpha_{t-1}} - \frac{\mu}{2}  =  \frac{g_{t, i}^2}{2\alpha} - \frac{\mu}{2} \leq 0.\notag
\end{align}

Thus,~\eqref{eq: sadam_main2} becomes
\begin{equation}
\sum_{t=1}^T f_t(x_t) - f_t(x) \leq \frac{\beta_1}{1-\beta_1} \langle m_T, x_T - x \rangle + \frac{1+\beta_1}{2(1-\beta_1)} \sum_{t=1}^T \alpha_t \| m_t \|^2_{\hat{v}_t^{-1}}. \notag
\end{equation}
We finalize by using $\frac{1+\beta_1}{2} \leq 1$, Lemma~\ref{lem: sadam_sum_grad_norms} for the last term, and $\|m_t\|_{\infty} \leq G$, $\| x_t - x \|_\infty \leq D$ for the first term
\begin{equation*}
\sum_{t=1}^T f_t(x_t) - f_t(x) \leq \frac{\beta_1 d G D}{1-\beta_1}  + \frac{\alpha}{1-\beta_1} \sum_{i=1}^d \log\left( \frac{\sum_{t=1}^T g_{t, i}^2}{\epsilon}+1 \right).
\qedhere
\end{equation*}
\end{proof}

\subsection{Proof for Zeroth order \adam}
We restate Proposition~\ref{prop: zero} and provide its proof.
\begin{repproposition}{prop: zero}
Assume that $f$ is convex, $L$-smooth, and $L_c$-Lipschitz, $\mathcal{X}$ is compact with diameter $D$.
Then ZO-AdaMM with $\beta_1, \beta_2 < 1$, $\gamma =\frac{
\beta_1^2}{\beta_2} < 1$ achieves
\begin{equation*}
\mathbb{E}\left[ \sum_{t=1}^T f_{t, \mu}(x_t) - f_{t, \mu}(x_\star) \right] \leq \frac{D^2\sqrt{T}}{2\a(1-\b_1)}\sum_{i=1}^d \mathbb{E}\left[ \hv^{1/2}_{T,i}\right] 
+\frac{\alpha\sqrt{1+\log T}}{\sqrt{(1-\beta_2)(1-\gamma)}} \sum_{i=1}^d
    \sqrt{\sum_{t=1}^T \mathbb{E}\left[ \hat{g}_{t,i}^2\right]}.
\end{equation*}
\end{repproposition}
\begin{proof}
We first note that ZO-AdaMM~\cite{chen2019zo} corresponds to using \amsgrad with $\hat{g}_t$ as the gradient input, rather than the true gradient $g_t$.
Therefore, we follow the proof structure of Theorem~\ref{th:amsgrad} with $\hat{g}_t$ as gradient input (instead of the true gradient $g_t$), until~\eqref{eq: ams_final_est}:
\begin{equation}
\sum_{t=1}^T \langle \hat{g}_t, x_t - x \rangle \leq \frac{D^2\sqrt{T}}{2\a(1-\b_1)}\sum_{i=1}^d \hv^{1/2}_{T,i} +\frac{\alpha\sqrt{1+\log T}}{\sqrt{(1-\beta_2)(1-\gamma)}} \sum_{i=1}^d
    \sqrt{\sum_{t=1}^T \hat{g}_{t,i}^2}\label{eq: zero_main1}
\end{equation}

With this bound in hand, we proceed as in the proof of~\citet[Proposition 4]{chen2019zo}. 
Specifically, note that $\mathbb{E}_t \left[ \hat{g}_t \right] = \nabla f_{t, \mu}(x_t)$ where the randomness is due to selection of the seed $\xi_t$ and the random vector $u$ in~\eqref{eq: zero_grad_def}.
Then, taking the full expectation and using convexity gives
\[\mathbb{E}\left[ \sum_{t=1}^T f_{t, \mu}(x_t) - f_{t, \mu}(x) \right] \leq \mathbb{E}\left[ \sum_{t=1}^T \langle \hat{g}_t, x_t - x \rangle \right].\]
Our claim then follows by applying Jensen's inequality,
after taking expectations in~\eqref{eq: zero_main1}.
\end{proof}

\subsection{Proofs for nonconvex \amsgrad}
\begin{lemma}(Bound for $\sum_{t=1}^T  \| \alpha_t \hat v_t^{-1/2} m_t \|^2$). \label{lem: nonconv bound}
Under \Cref{as:2}, $\beta_1<1$, $\beta_2 < 1$, $\gamma =
\frac{\beta_1^2}{\beta_2} < 1$, and the definitions of $\alpha_t$, $m_t$, $v_t$, $\hat v_t$ in \amsgrad, it holds that
\begin{equation}\label{eq: nonconvex_est1}
\sum_{t=1}^T \| \alpha_t \hat{v}_t^{-1/2}m_t \|^2 \leq \frac{d(1-\b_1)^2\a^2(1+\log T)}{(1-\beta_2)(1-\gamma)}.
\end{equation}
\end{lemma}
\begin{proof}
We first note the inequality for positive numbers
\begin{equation}
\frac{(a_1+\dots+a_t)^2}{b_1+\dots+b_t} \leq \frac{a_1^2}{b_1}+ \dots + \frac{a_t^2}{b_t},\label{eq: new_ineq}
\end{equation}
which is a consequence of Cauchy-Schwarz inequality.

Now we have 
\begin{align}
 \| \alpha_t\hat{v}_t^{-1/2} m_{t}  \|^2 &=  \sum_{i=1}^d \alpha_t^2 \frac{m_{t,i}^2}{\hat{v}_{t,i}} \leq \sum_{i=1}^d \alpha_t^2 \frac{m_{t,i}^2}{{v}_{t,i}}\notag \\
&= \sum_{i=1}^d \alpha_t^2 \frac{\left( \sum_{j=1}^t(1-\beta_1)\beta_1^{t-j}{g}_{j,i}\right)^2}{\sum_{j=1}^t(1-\beta_2) \beta_2^{t-j} g_{j, i}^2}\notag \\
&=\frac{(1-\beta_1)^2}{1-\beta_2} \sum_{i=1}^d \alpha_t^2 \frac{\left( \sum_{j=1}^t\beta_1^{t-j}{g}_{j,i}\right)^2}{\sum_{j=1}^t \beta_2^{t-j} g_{j, i}^2}\notag \\
&\leq\frac{(1-\beta_1)^2}{1-\beta_2} \sum_{i=1}^d \alpha_t^2 \sum_{j=1}^t\frac{\beta_1^{2(t-j)}{g}_{j,i}^2}{ \beta_2^{t-j} g_{j, i}^2}=\frac{(1-\beta_1)^2}{1-\beta_2}\sum_{i=1}^d \alpha_t^2 \sum_{j=1}^t\gamma^{t-j}\notag \\
&\leq\frac{d(1-\beta_1)^2}{(1-\beta_2)(1-\gamma)} \alpha_t^2, \label{eq: nc_ams1_last2}
\end{align}
where the first inequality uses $\hat{v}_{t, i} \geq v_{t, i}$, and the second equality uses the expressions from~\eqref{eq: mt_vt}.
The second inequality is by~\eqref{eq: new_ineq}, and the final one by
the sum of geometric series with $\gamma=\frac{\b_1^2}{\b_2}$. Since $\a_t^2
=\frac{\a^2}{t}$, the final inequality~\eqref{eq: nonconvex_est1} follows.
\end{proof}

The reader could notice that all proofs so far were based on
\Cref{lem: decop}. In fact, we can formulate a more general statement, which will
be the key in the nonconvex settings.
\begin{lemma}\label{lem: nonconv_decomp}
Let $m_t = \beta_1 m_{t-1} + (1-\beta_1) g_t$ and
$A_t\in \mathbb{R}^d$, $\forall t=1,\dots,T$. Then it follows that
\begin{equation}
\langle A_t, g_t \rangle = \frac{1}{1-\beta_1}\bigg( \langle A_t, m_t \rangle - \langle A_{t-1}, m_{t-1} \rangle \bigg) + \langle A_{t-1}, m_{t-1} \rangle + \frac{\beta_1}{1-\beta_1} \langle A_{t-1} - A_{t}, m_{t-1} \rangle.\label{eq: nc_lemma}
\end{equation}
\end{lemma}
For convex case, we plugged in $A_t = x_t - x$, while  for the
nonconvex case we will use $A_t = \alpha_t \hat{v}_t^{-1/2}\nabla
f(x_t)$. Obviously, its proof relies on the same algebra as in~\Cref{lem: decop}.

We move onto restating Theorem~\ref{th: nonconvex} and presenting its proof.
\begin{reptheorem}{th: nonconvex}
Under \Cref{as:2}, $\beta_1<1$, $\beta_2 < 1$, and $\gamma =
\frac{\beta_1^2}{\beta_2} < 1$ \amsgrad achieves
\begin{multline*}
\frac{1}{T} \sum_{t=1}^T \mathbb{E}\left[ \| \nabla f(x_t) \|^2 \right] \leq \frac{1}{\sqrt{T}}\bigg[ \frac{G}{\alpha}\left(f(x_1) - f(x_\star)\right)+ \frac{G^3}{(1-\beta_1)}\| \hat{v}_0^{-1/2} \|_1+ \frac{G^3d}{4L\alpha(1-\beta_1)} \\ 
+\frac{GLd\alpha(1-\beta_1)(1+\log T)}{(1-\beta_2)(1-\gamma)}\bigg].
\end{multline*}
\end{reptheorem}
\begin{proof}
    Let $A_t = \alpha_t \hat{v}_t^{-1/2}\nabla f(x_t) $ for $t\geq 1$
    and $A_0=A_1$. By summing~\eqref{eq: nc_lemma} over $t=1,\dots, T$ and using that
$m_0=0$, $\langle A_0, m_{0} \rangle =0$, $\langle A_1 - A_{0}, m_{0}
\rangle =0$, we obtain
\begin{align}\label{eq:nonconv main1}
  \sum_{t=1}^T\lr{A_t, g_t} &= \frac{1}{1-\b_1}\lr{A_T,m_T} +
  \sum_{t=1}^{T-1}\lr{A_t,m_t} + \frac{\b_1}{1-\b_1}\sum_{t=1}^{T}\lr{A_{t-1}-A_{t},m_{t-1}}\notag\\&=\frac{\b_1}{1-\b_1}\lr{A_T,m_T} +
  \sum_{t=1}^{T}\lr{A_t,m_t} + \frac{\b_1}{1-\b_1}\sum_{t=1}^{T-1}\lr{A_{t}-A_{t+1},m_{t}}.
\end{align}
We are going to derive bounds for \eqref{eq:nonconv main1} and
then take expectation to get an estimate for
$\mathbb{E}\big[\n{\nabla f(x_t)}^2\big]$.
    
To this end, we note that for the expectation conditioned on the
history until selecting $g_t$, one has
$\mathbb{E}_t [g_t] = \nabla f(x_t)$, since under this condition
$\hat{v}_{t-1}$ is deterministic as it does not depend on $g_{t}$. It is tempting to compute
$\mathbb{E}_t\big[\n{\nabla f(x_t)}^2\big]$ by using
$\mathbb{E}_t\big[\langle A_t, g_t\rangle\big] =
\mathbb{E}_t\big[\langle \alpha_t \hat{v}_t^{-1/2}\nabla f(x_t),
g_t\rangle \big]$.  Unfortunately, this is not feasible, as $\hat{v_t}$
does depend on $g_t$. Instead, we bound $\langle A_t, g_t\rangle$
from below by a more suitable random variable for taking conditional expectation $\mathbb{E}_t$ .

    $\bullet$ \emph{Bound for} $\langle A_{t}, g_{t} \rangle$

First, we note
\begin{align}\label{eq:nonconv A_t g_t}
\langle A_t, g_t \rangle &= \langle \alpha_t \hat{v}_t^{-1/2}\nabla f(x_t), g_t \rangle = \langle \alpha_{t-1} \hat{v}_{t-1}^{-1/2}\nabla f(x_t), g_t \rangle - \langle \nabla f(x_t), \big(\alpha_{t-1}\hat{v}_{t-1}^{-1/2}-\alpha_t \hat{v}_t^{-1/2}\big) g_t \rangle.
\end{align}
To simplify derivations, we set $\a_0  = \a = \a_1$.  Now, for the last
term in the right-hand side we have
\begin{align} 
\langle \nabla f(x_t), \big(\alpha_{t-1}\hat{v}_{t-1}^{-1/2}-\alpha_t \hat{v}_t^{-1/2}\big)g_t \rangle &\leq \| \nabla f(x_t) \|_{\infty} \|  \alpha_{t-1} \hat{v}_{t-1}^{-1/2} - \alpha_t \hat{v}_t^{-1/2}\|_1 \| g_{t} \|_{\infty} \notag \\
&\leq G^2\left( \| \alpha_{t-1} \hat{v}_{t-1}^{-1/2} \|_1 - \| \alpha_{t} \hat{v}_{t}^{-1/2} \|_1 \right),\label{eq: l1_dif_term}
\end{align}
where we used H\"older's inequality, and
$\alpha_{t-1}\hat{v}_{t-1,i}^{-1/2}\geq \alpha_t \hat{v}_{t,i}^{-1/2}$ (note
that for $t=1$, this is still true, since $\hv_1\geq \hv_0$ and $\a_0=\a_1$).  Combining
\eqref{eq: l1_dif_term} and \eqref{eq:nonconv A_t g_t} yields
\begin{align}\label{eq:nonconv bound1}
\langle A_t, g_t \rangle \geq \langle \alpha_{t-1} \hat{v}_{t-1}^{-1/2}\nabla f(x_t), g_t \rangle - G^2(\| \alpha_{t-1} \hat{v}_{t-1}^{-1/2} \|_1 - \| \alpha_t \hat{v}_t^{-1/2} \|_1).
\end{align}
Clearly, the term
$\langle \alpha_{t-1} \hat{v}_{t-1}^{-1/2}\nabla f(x_t), g_t \rangle$
is more convenient for taking $\mathbb{E}_t$. We will do it right
after we bound the right-hand side of~\eqref{eq:nonconv main1}.  Let us
focus on each term of~\eqref{eq:nonconv main1} separately.
 
$\bullet$ \emph{Bound for $\langle A_t - A_{t+1}, m_{t}\rangle$}

We rearrange terms to obtain
\begin{align}\label{eq:nonconv_At}
\langle A_{t} - A_{t+1}, &m_{t}\rangle = \langle \alpha_{t}\hat{v}_{t}^{-1/2}\nabla f(x_{t})- \alpha_{t+1} \hat{v}_{t+1}^{-1/2}\nabla f(x_{t+1}), m_{t} \rangle \notag \\
&=\langle \alpha_{t} \hat{v}_{t}^{-1/2} \nabla f(x_{t+1}) -\alpha_{t+1} \hat{v}_{t+1}^{-1/2}\nabla f(x_{t+1}), m_{t} \rangle + \langle  \alpha_{t}\hat{v}_{t}^{-1/2}\nabla f(x_{t})-\alpha_{t}\hat{v}_{t}^{-1/2}\nabla f(x_{t+1}), m_{t} \rangle\notag \\
&=\langle \nabla f(x_{t+1}),\big(\alpha_{t}\hat{v}_{t}^{-1/2}-\alpha_{t+1} \hat{v}_{t+1}^{-1/2}\big) m_{t} \rangle + \langle \nabla f(x_{t})-\nabla f(x_{t+1}), \alpha_{t} \hat{v}_{t}^{-1/2}m_{t} \rangle.
\end{align}
For the first term we use almost the same inequality as in \eqref{eq: l1_dif_term}
\begin{align*} 
\langle \nabla f(x_{t+1}), (\alpha_{t}\hat{v}_{t}^{-1/2}-\alpha_{t+1}\hat{v}_{t+1}^{-1/2})m_{t} \rangle &\leq \| \nabla f(x_{t+1}) \|_{\infty} \|\alpha_{t} \hat{v}_{t}^{-1/2}-\alpha_{t+!}\hat{v}_{t+1}^{-1/2} \|_1 \| m_{t} \|_{\infty} \notag \\
&\leq G^2\left( \| \alpha_{t} \hat{v}_{t}^{-1/2} \|_1 - \| \alpha_{t+1} \hat{v}_{t+1}^{-1/2} \|_1 \right).
\end{align*}
For the second term we use smoothness of $f$ and the update rule for
$x_{t+1}$
\begin{align*}
\langle \nabla f(x_t) - \nabla f(x_{t+1}), \alpha_{t} \hat{v}_{t}^{-1/2}m_{t} \rangle &\leq \|\nabla f(x_t) - \nabla f(x_{t+1}) \| \| \alpha_{t}\hat{v}_{t}^{-1/2}m_{t} \| \notag \\
&\leq L \| x_{t+1} - x_{t} \| \| \alpha_{t}\hat{v}_{t}^{-1/2}m_{t} \| = L \| x_{t+1} - x_{t} \|^2.
\end{align*}
We apply above estimates in \eqref{eq:nonconv_At} to derive
\begin{equation}\label{eq:nonconv bound2}
\langle A_t - A_{t+1}, m_{t} \rangle \leq G^2 \left( \| \alpha_{t}\hat{v}_{t}^{-1/2} \|_1 - \| \alpha_{t+1} \hat{v}_{t+1}^{-1/2}\|_1 \right) + L \| x_{t+1} - x_{t}\|^2.
\end{equation}
 
$\bullet$ \emph{Bound for} $\langle A_{t}, m_{t} \rangle$

By the update of $x_{t+1}$ and the descent lemma, we have
\begin{align}\label{eq:nonconv bound3}
\langle A_{t}, m_{t} \rangle &=  \langle \alpha_{t}\hat{v}_{t}^{-1/2}\nabla f(x_{t}),m_{t} \rangle = \langle \nabla f(x_{t}), \alpha_{t}\hat{v}_{t}^{-1/2}m_{t} \rangle\notag\\&= \langle \nabla f(x_{t}), x_{t} - x_{t+1} \rangle \leq f(x_{t}) - f(x_{t+1}) + \frac{L}{2} \| x_{t+1} - x_{t} \|^2.
\end{align}

$\bullet$ \emph{Final bounds}
 
Combining the bounds, we obtain 
\begin{align}\label{eq:nonconv RHS1}
\text{RHS of \eqref{eq:nonconv main1}} &\leq
  \frac{\b_1}{1-\b_1}\lr{A_T,m_T} + \big(f(x_1)-f(x_{T+1}) +\frac L 2
  \sum_{t=1}^T \n{x_{t+1}-x_t}^2\big) \notag\\&\quad +\frac{\b_1G^2}{1-\b_1}\big(\n{\a_1\hv_1^{-1/2}}_1-\n{\a_T\hv_T^{-1/2}}_1\big)+\frac{\b_1L}{1-\b_1}\sum_{t=1}^{T-1}\n{x_{t+1}-x_t}^2.
\end{align}
By Young's inequality, $x_T - x_{T+1} =\alpha_T\hat{v}_T^{-1/2} m_T$,
and $\n{\nabla f(x_T)}_{\infty}\leq G$,
\begin{align*}
  \lr{A_T,m_T} = \langle \nabla f(x_T), \alpha_T\hat{v}_T^{-1/2} m_T
  \rangle \leq L \| \alpha_T \hat{v}_T^{-1/2} m_T \|^2 + \frac{1}{4L}
  \| \nabla f(x_T) \|^2\leq L \| x_{T+1}-x_T \|^2 + \frac{G^2d}{4L}.
\end{align*}
Hence, we can conclude in \eqref{eq:nonconv RHS1} 
 \begin{align}\label{eq:nonconv RHS2} 
   \text{RHS of \eqref{eq:nonconv main1}} &\leq
   \frac{\b_1 G^2 d}{4(1-\b_1)L} + \big(f(x_1)-f(x_{T+1})
                                            +\frac{L}{2}\sum_{t=1}^T
                                            \n{x_{t+1}-x_t}^2\big)+\frac{\b_1
                                            G^2}{1-\b_1}\n{\a_1
                                             \hv_1^{-1/2}}_1+\frac{\b_1L}{1-\b_1}\sum_{t=1}^{T}\n{x_{t+1}-x_t}^2\notag \\&\leq\frac{\b_1G^2d}{4(1-\b_1)L}+\big(f(x_1)-f(x_{\star})\big)+\frac{\a\b_1G^2}{1-\b_1}\n{\hv_1^{-1/2}}_1+\frac{L}{1-\b_1}\sum_{t=1}^{T}\n{x_{t+1}-x_t}^2\notag\\&\leq\frac{\b_1G^2d}{4(1-\b_1)L}+\big(f(x_1)-f(x_{\star})\big)+\frac{\a\b_1G^2}{1-\b_1}\n{\hv_1^{-1/2}}_1+\frac{dL\a^2(1-\b_1)(1+\log T)}{(1-\b_2)(1-\gamma)},
 \end{align} 
 where in the second inequality we used $f(x_T)\geq f(x_\star)$,
$\a_1=\a$, and $\frac{1+\beta_1}{2} \leq 1$ and the final inequality
follows from \Cref{lem: nonconv bound}, as
$\sum_{t=1}^{T}\n{x_{t+1}-x_t}^2 =
\sum_{t=1}^T\n{\a_t\hv_t^{-1/2}m_t}^2$.

Now we analyze the left-hand side of~\eqref{eq:nonconv main1}.  Using~\eqref{eq:nonconv bound1}, we deduce
\begin{align}\label{eq:nonconv LHS1}
\text{LHS of \eqref{eq:nonconv main1}} &\geq \sum_{t=1}^T\langle 
  \alpha_{t-1} \hat{v}_{t-1}^{-1/2}\nabla f(x_t), g_t \rangle - G^2\big(\|
  \alpha_{0} \hat{v}_{0}^{-1/2} \|_1 - \|\alpha_{T} \hat{v}_{T}^{-1/2}
  \|_1\big) \notag\\ & \geq \sum_{t=1}^T\langle 
  \alpha_{t-1} \hat{v}_{t-1}^{-1/2}\nabla f(x_t), g_t \rangle - \a G^2\|
   \hat{v}_{0}^{-1/2} \|_1,
\end{align}
where we used $\a_0=\a$ and $\|\alpha_{T} \hat{v}_{T}^{-1/2}
  \|_1\geq 0$.
   
  Finally, combining \eqref{eq:nonconv RHS2}, \eqref{eq:nonconv LHS1},
  and \eqref{eq:nonconv main1}, we arrive at 
\begin{align}\label{eq:nonconv before exp}
  \sum_{t=1}^T\langle \alpha_{t-1} \hat{v}_{t-1}^{-1/2}\nabla f(x_t),
  g_t \rangle &\leq f(x_1) - f(x_\star)
                + \frac{\a\b_1G^2}{1-\b_1}\n{\hv_1^{-1/2}}_1+\a G^2\n{\hv_0^{-1/2}}_1\notag \\& \quad  + \frac{\b_1G^2d}{4(1-\b_1)L}+\frac{dL\alpha^2(1-\b_1)(1+\log
  T)}{(1-\b_2)(1-\gamma)}\notag \\&\leq  f(x_1) - f(x_\star)
                + \frac{\a G^2}{1-\b_1}\n{\hv_0^{-1/2}}_1   + \frac{\b_1G^2d}{4(1-\b_1)L}+\frac{dL\a^2(1-\b_1)(1+\log
  T)}{(1-\b_2)(1-\gamma)},
\end{align}
where we used that $\hv_{0,i}^{-1/2}\geq \hv_{1,i}^{-1/2}$.

Since $\mathbb{E}_t$ is conditioned on the history until selecting
$g_t$, $\hat v_{t-1}$ does not depend on $g_t$, $\mathbb{E}_t [g_t] = \nabla f(x_t)$, and
$\|\hat{v}_t\|_\infty \leq G^2$, we obtain
\begin{align*}
\mathbb{E}_{t} \left[ \langle \alpha_{t-1}\hat{v}_{t-1}^{-1/2} \nabla
    f(x_t), g_t \rangle \right]= \langle \alpha_{t-1}\hat{v}_{t-1}^{-1/2} \nabla f(x_t), \nabla f(x_t) \rangle = \sum_{i=1}^d \frac{\alpha_{t-1}}{\hat{v}_{t-1, i}^{1/2}} (\nabla f(x_t))_i^2 \geq \frac{\alpha}{\sqrt{T} G} \| \nabla f(x_t) \|^2.
\end{align*}
Taking the full expectation above yields
\begin{align*}
\mathbb{E} \left[ \langle \alpha_{t-1}\hat{v}_{t-1}^{-1/2} \nabla
    f(x_t), g_t \rangle \right] \geq  \frac{\alpha}{\sqrt{T} G} \mathbb{E}\left[\| \nabla f(x_t) \|^2\right].
\end{align*} 
Thus, by taking the full expectation in \eqref{eq:nonconv before exp}, we
deduce 
\begin{align*}
\frac{\alpha}{\sqrt{T}G} \sum_{t=1}^T \mathbb{E}\left[ \| \nabla
  f(x_t) \|^2 \right] &\leq f(x_1) - f(x_\star) + \frac{\a G^2}{1-\beta_1}\| \hat{v}_0^{-1/2} \|_1+ \frac{G^2d}{4L(1-\beta_1)} +\frac{dL\alpha^2(1-\beta_1)(1+\log T)}{(1-\beta_2)(1-\gamma)},
\end{align*}
from which the final bound follows immediately.
\end{proof}

\end{document}